\documentclass{article}

\usepackage[accepted]{icml2025}

\usepackage{microtype}
\usepackage{graphicx}
\usepackage{subfigure}

\usepackage{natbib}

\usepackage[utf8]{inputenc}
\usepackage{amsmath,amsfonts}
\usepackage{amssymb}
\usepackage{graphicx}
\usepackage{tikz}
\usepackage{mathtools}
\usepackage{amsthm}
\usepackage{nccmath}
\usepackage{xspace}
\usepackage{xcolor}
\usepackage{booktabs}
\usepackage{array}
\usepackage{subcaption}
\definecolor{MyBlue}{rgb}{0.12, 0.25, 0.67}
\definecolor{sbBlue}{RGB}{31,119,180}
\usepackage[colorlinks,allcolors=MyBlue]{hyperref}
\usepackage{pifont}

\usepackage[group-separator={,}]{siunitx}
\usepackage{url}
\usepackage{enumitem}
\usepackage{pdflscape}
\usepackage{verbatim}
\usepackage{arydshln}
\usepackage{multirow}
\usepackage{bigdelim}
\usepackage{float}
\usepackage{pgfplots}
\usepackage{cleveref}
\usepackage[font=small]{caption}

\usetikzlibrary{shadows,arrows,decorations,decorations.shapes,backgrounds,shapes,snakes,automata,fit,petri,shapes.multipart,calc,positioning,shapes.geometric,graphs,graphs.standard,plotmarks,math,arrows.meta}
\usepackage{tikz-qtree}

\theoremstyle{plain}
\newtheorem{theorem}{Theorem}[section]
\newtheorem{proposition}[theorem]{Proposition}
\newtheorem{lemma}{Lemma}[section]

\theoremstyle{definition}
\newtheorem{definition}[theorem]{Definition}

\theoremstyle{remark}

\usepackage{thmtools}
\usepackage{thm-restate}

\usepackage{newfloat}
\DeclareFloatingEnvironment[
    fileext=los,
    listname={List of Constructions},
    name=Construction,
    placement=Htbhp,
]{construction}

\renewcommand{\algorithmicendfor}{\vspace{-1.2em}}%
\renewcommand{\algorithmicendif}{\vspace{-1.2em}}%
\renewcommand{\algorithmicendfunction}{\vspace{-1.2em}}%

\DeclareMathOperator*{\argmax}{arg\,max}
\DeclareMathOperator*{\E}{\mathbb{E}}
\DeclareMathOperator*{\argmin}{arg\,min}
\DeclareMathOperator*{\len}{len}

\DeclareMathOperator*{\vol}{vol}
\DeclareMathOperator*{\diam}{diam}
\allowdisplaybreaks 

\newcommand\bbr{\mathbb{R}}

\newcommand\bbrspos{\mathbb{R}_{> 0}}
\newcommand\bbn{\mathbb{N}}
\newcommand\ep{\varepsilon}
\newcommand\hata{\tilde{y}}

\newcommand\x{\boldsymbol{a}}
\newcommand\A{\boldsymbol{y}}
\newcommand\s{\boldsymbol{x}}
\newcommand\X{\mathcal{X}}
\newcommand\Y{\mathcal{Y}}

\newcommand\mufx{\mu_{f,\x}}
\newcommand\bfone{\mathbf{1}}
\newcommand\B{\mathcal{B}}

\newcommand\D{\mathcal{D}}
\newcommand\norm[1]{||#1||}

\newcommand\hats{\hat{X}}
\newcommand\M{\mathcal{M}}
\newcommand\bfmu{\boldsymbol{\mu}}
\newcommand\tilpi{\tilde{\Pi}}
\newcommand\stilpi{\tilde{\pi}}

\newcommand\bfd{\boldsymbol{\mathcal{D}}}

\newcommand\mum{\mu_0^m}
\newcommand\rplus{R_T^{+}}
\newcommand\rmul{R_T^{\times}}

\icmltitlerunning{Avoiding Catastrophe in Online Learning by Asking for Help}

\begin{document}

\twocolumn[
\icmltitle{Avoiding Catastrophe in Online Learning by Asking for Help}


\icmlsetsymbol{equal}{*}

\begin{icmlauthorlist}
\icmlauthor{Benjamin Plaut}{ucb}
\icmlauthor{Hanlin Zhu}{ucb}
\icmlauthor{Stuart Russell}{ucb}
\end{icmlauthorlist}

\icmlaffiliation{ucb}{University of California, Berkeley}

\icmlcorrespondingauthor{Benjamin Plaut}{plaut@berkeley.edu}
\icmlkeywords{online learning, AI safety, asking for help, irreversibility}

\vskip 0.3in
]

\printAffiliationsAndNotice{}


\setcounter{footnote}{0} 

\begin{abstract}
Most learning algorithms with formal regret guarantees assume that all mistakes are recoverable and essentially rely on trying all possible behaviors. This approach is problematic when some mistakes are \emph{catastrophic}, i.e., irreparable. We propose an online learning problem where the goal is to minimize the chance of catastrophe. Specifically, we assume that the payoff in each round represents the chance of avoiding catastrophe in that round and try to maximize the product of payoffs (the overall chance of avoiding catastrophe) while allowing a limited number of queries to a mentor. We also assume that the agent can transfer knowledge between similar inputs. We first show that in general, any algorithm either queries the mentor at a linear rate or is nearly guaranteed to cause catastrophe. However, in settings where the mentor policy class is learnable in the standard online model, we provide an algorithm whose regret and rate of querying the mentor both approach 0 as the time horizon grows. Although our focus is the product of payoffs, we provide matching bounds for the typical additive regret. Conceptually, if a policy class is learnable in the absence of catastrophic risk, it is learnable in the presence of catastrophic risk if the agent can ask for help.\looseness=-1
\end{abstract}

\section{Introduction}\label{sec:intro}

There has been mounting concern over catastrophic risk from AI, including but not limited to autonomous weapon accidents \citep{Abaimov2020}, bioterrorism \citep{mouton2024operational}, cyberattacks on critical infrastructure \citep{guembe_emerging_2022}, and loss of control \citep{bengio2024managing}. See \citet{critch2023tasra} and \citet{hendrycks2023overviewcatastrophicairisks} for taxonomies of societal-scale AI risks. In this paper, we use ``catastrophe'' to refer to any kind of irreparable harm. In addition to the large-scale risks above, our definition also covers smaller-scale (yet still unacceptable) incidents such as serious medical errors \citep{rajpurkar2022ai}, crashing a robotic vehicle \citep{kohli2020enabling}, and discriminatory sentencing \citep{villasenor2020artificial}. 

The gravity of these risks contrasts starkly with the dearth of theoretical understanding of how to avoid them. Nearly all of learning theory explicitly or implicitly assumes that no single mistake is too costly. We focus on \emph{online learning}, where an agent repeatedly interacts with an unknown environment and uses its observations to gradually improve its performance. Most online learning algorithms essentially try all possible behaviors and see what works well. We do not want autonomous weapons or surgical robots to try all possible behaviors.\looseness=-1


More precisely, trial-and-error-style algorithms only work when catastrophe is assumed to be impossible. This assumption can take multiple forms, such as that the agent's actions do not affect future inputs (e.g., \citealp{slivkins_contextual_2011}), that no action has irreversible effects (e.g., \citealp{jaksch_near-optimal_2010}) or that the environment is reset at the start of each ``episode'' (e.g., \citealp{azar_minimax_2017}). One could train an agent entirely in a controlled lab setting where one of those assumptions does hold, but we argue that sufficiently general agents will inevitably encounter novel scenarios when deployed in the real world. Machine learning models often behave unpredictably in unfamiliar environments (see, e.g., \citealp{quinonero2022dataset}), and we do not want AI biologists or robotic vehicles to behave unpredictably.


The goal of this paper is to understand the conditions under which it is possible to formally guarantee avoidance of catastrophe in online learning. Certainly some conditions are necessary, because the problem is hopeless if the agent must rely purely on trial-and-error: any untried action could lead to paradise or disaster and the agent has no way to predict which. In the real world, however, one need not learn through pure trial-and-error: one can also ask for help. We think it is critical for high-stakes AI applications to employ a designated supervisor who can be asked for help. Examples include a human doctor supervising AI doctors, a robotic vehicle with a human driver who can take over in emergencies, autonomous weapons with a human operator, and many more. We hope that our work constitutes a step in the direction of practical safety guarantees for such applications.\looseness=-1

\subsection{Our model}

We propose an online learning model of avoiding catastrophe with mentor help. On each time step, the agent observes an input, selects an action or queries the mentor, and obtains a payoff. Each payoff represents the probability of avoiding catastrophe on that time step (conditioned on no prior catastrophe). The agent's goal is to maximize the \emph{product} of payoffs, which is equal to the overall probability of avoiding catastrophe.\footnote{Conditioning on no prior catastrophe means we do not need to assume that these probabilities are independent (and if
catastrophe has already occurred, this time step does not matter). This is due to the chain rule of probability.} As is standard in online learning, we consider the product of payoffs obtained while learning, not the product of payoffs of some final policy.



The (possibly suboptimal) mentor has a stationary policy, and when queried, the mentor illustrates their policy's action for the current input. We want the agent's regret -- defined as the gap between the mentor's performance and the agent's performance -- to go to zero as the time horizon $T$ grows. In other words, with enough time, the agent should avoid catastrophe nearly as well as the mentor. We also expect the agent to become self-sufficient over time: the number of queries to the mentor should be sublinear in $T$, or equivalently, the rate of querying the mentor should go to zero.

\subsection{Our assumptions}

The agent needs some way to make inferences about unqueried inputs in order to decide when to ask for help. Much past work has used Bayesian inference, which suffers from tractability issues in complex environments.\footnote{For the curious reader, \citet{betancourt_conceptual_2018} provides a thorough treatment. See also Section~\ref{sec:related}.} We instead assume that the mentor policy satisfies a novel property that we call \emph{local generalization}: informally, if the mentor told us that an action was safe for a similar input, then that action is probably also safe for the current input. For example, if it is safe to ignore a 3 mm spot on an X-ray, it is likely (but not certainly) also safe to ignore a 3.1 mm spot with the same density, location, etc. Unlike Bayesian inference, local generalization only requires computing distances and is compatible with any input space that admits a distance metric. See \Cref{sec:lg} for further discussion of local generalization. \looseness=-1

Unlike the standard online learning model, we assume that the agent does not observe payoffs. This is because the payoff in our model represents the chance of avoiding catastrophe on that time step. In the real world, one only observes whether catastrophe occurred, not its probability.\footnote{One may be able to detect ``close calls'' in some cases, but observing the precise probability seems unrealistic.}

\subsection{Standard online learning}\label{sec:online}

To properly understand our results, it is important to understand standard online learning. In the standard model, the agent observes an input on each time step and must choose an action. An adversary then reveals the correct action, which results in some payoff to the agent. The goal is sublinear regret with respect to the sum of payoffs, or equivalently, the average regret per time step should go to 0 as $T \to \infty$. \Cref{fig:compare} delineates the precise differences between the standard model and our model.

If the adversary's choices are unconstrained, the problem is hopeless: if the adversary determines the correct action on each time step randomly and independently, the agent can do no better than random guessing. However, sublinear regret becomes possible if (1) the hypothesis class has finite Littlestone dimension \citep{littlestone1988learning}, or (2) the hypothesis class has finite VC dimension \citep{vapnik_uniform_1971} and the input is \emph{$\sigma$-smooth}\footnote{Informally, the adversary chooses a distribution over inputs instead of a precise input. See \Cref{sec:model} for the formal definition.} \citep{haghtalab2024smoothed}.

The goal of sublinear regret in online learning implicitly assumes catastrophe is impossible: the agent can make arbitrarily many (and arbitrarily costly) mistakes as long as the \emph{average} regret per time step goes to 0. In contrast, we demand \emph{subconstant} regret: the \emph{total} probability of catastrophe should go to 0. Furthermore, standard online learning allows the agent to observe payoffs on every time step, while our agent only receives feedback on time steps with queries. However, the combination of a mentor and local generalization allows our agent to learn without trying actions directly, which is enough to offset all of the above disadvantages. 

\begin{table}
\caption{Comparison between the standard online learning model and our model.}
\label{fig:compare}
\centering
\renewcommand{\arraystretch}{1.3}
\resizebox{\linewidth}{!}{%
\begin{tabular}{c|cc}
& Standard model & Our model\\
\hline
Objective & Sum of payoffs & Product of payoffs\\
Regret goal & Sublinear & Subconstant\\
Feedback & Every time step & Only from queries\\
Mentor & No & Yes\\
\hspace{-.1 in} Local generalization & No & Yes
\end{tabular}
}
\renewcommand{\arraystretch}{1}
\end{table}

\subsection{Our results}\label{sec:results}

At a high level, we show that avoiding catastrophe with the help of a mentor and local generalization is no harder than online learning without catastrophic risk.

We first show that in general, any algorithm with sublinear queries to the mentor has unbounded regret in the worst-case (\Cref{thm:neg}). As a corollary, even when the mentor can avoid catastrophe with certainty, any algorithm either needs extensive supervision or is nearly guaranteed to cause catastrophe (\Cref{cor:neg-prob}).



Our primary result is a simple algorithm whose total regret and rate of querying the mentor both go to 0 as $T \to\infty$ when either (1) the mentor policy class has finite Littlestone dimension or (2) the mentor policy class has finite VC dimension and the input sequence is $\sigma$-smooth (\Cref{thm:pos-nd}). Conceptually, the algorithm has two components: (1) for ``in-distribution'' inputs, run a standard online learning algorithm  (adjusted to account for only receiving feedback in response to queries), and (2) for ``out-of-distribution'' inputs, ask for help. Our algorithm can handle an unbounded input space and does not need to know the local generalization constant.\looseness=-1


Although we focus on the product of payoffs, we show that the results above (both positive and negative) hold for the typical additive regret as well. In fact, we show that multiplicative regret and additive regret are tightly related in our setting (\Cref{lem:prod-vs-add}).

In summary, the combination of local generalization and a mentor allows us to reduce the regret by an entire factor of $T$, resulting in subconstant regret (multiplicative or additive) instead of the typical sublinear regret.






\section{Related work}\label{sec:related}

\textbf{Learning with irreversible costs.} Despite the ubiquity of irreversible costs in the real world, theoretical work on this topic remains limited. This may be due to the fundamental modeling question of how the agent should learn about novel inputs or actions without trying them directly.\looseness=-1

The most common approach is to allow the agent to ask for help. This alone is insufficient, however: the agent must have some way to decide \emph{when} to ask for help. A popular solution is to perform Bayesian inference on the world model, but this has two tricky requirements: (1) a prior distribution which contains the true world model (or an approximation), and (2) an environment where computing (or approximating) the posterior is tractable. A finite set of possible environments satisfies both conditions but is unrealistic in many real-world scenarios. In contrast, our algorithm can handle an uncountable policy class and a continuous unbounded input space, which is crucial for many real-world scenarios in which one never sees the exact same input twice.\looseness=-1

Bayesian inference combined with asking for help is studied by \citet{cohen_curiosity_2021, cohen_pessimism_2020, kosoy_delegative_2019, mindermann_active_2018}. We also mention \citet{hadfield-menell_inverse_2017, moldovan_safe_2012, turchetta_safe_2016}, who utilize Bayesian inference in the context of safe (online) reinforcement learning without asking for help (and without regret bounds).

We are only aware of two papers that theoretically address irreversibility without Bayesian inference: \citet{grinsztajn_there_2021} and \citet{maillard_active_2019}. The former proposes to sample trajectories and learn reversibility based on temporal consistency between states: intuitively, if $s_1$ always precedes $s_2$, we can infer that $s_1$ is unreachable from $s_2$. Although the paper theoretically grounds this intuition, there is no formal regret guarantee. The latter presents an algorithm which asks for help in the form of rollouts from the current state. However, the regret bound and number of rollouts are both linear in the worst case, due to the dependence on the $\gamma^*$ parameter which roughly captures how bad an irreversible action can be. In contrast, our algorithm achieves good regret even when actions are maximally bad.\looseness=-1

To our knowledge, we are the first to provide an algorithm which formally guarantees avoidance of catastrophe (with high probability) without Bayesian inference. We are also not aware of prior results comparable to our negative result, including in the Bayesian regime.

\textbf{Safe reinforcement learning (RL).} The safe RL problem is typically formulated as a constrained Markov Decision Process (CMDP) \citep{altman2021constrained}. In CMDPs, the agent must maximize reward while also satisfying safety constraints. See \citet{gu_review_2023, zhao_state-wise_2023,wachi_survey_2024}
 for surveys. The two most relevant safe RL papers are \citet{liu2021learning} and \citet{stradi2024learning}, both of which provide algorithms guaranteed to satisfy initially unknown safety constraints. Since neither paper allows external help, they require strong assumptions to make the problem tractable: the aforementioned results assume that the agent (1) knows a strictly safe policy upfront (i.e., a policy which satisfies the safety constraints with slack), (2) is periodically reset, and (3) observes the safety costs. In contrast, our agent has no prior knowledge, is never reset, and never observes payoffs. \looseness=-1

\textbf{Online learning.} See \citet{cesa2006prediction} and Chapter 21 of \citet{shalev-shwartz_understanding_2014} for introductions to online learning. A classical result states that sublinear regret is possible if and only if the hypothesis class has finite Littlestone dimension \citep{littlestone1988learning}. However, even some simple hypothesis classes have infinite Littlestone dimension, such as the class of thresholds on $[0,1]$ (Example 21.4 in \citealp{shalev-shwartz_understanding_2014}). Recently, \citet{haghtalab2024smoothed} showed that if the adversary only chooses a distribution over inputs rather than the precise input, only finite VC dimension \citep{vapnik_uniform_1971} is needed for sublinear regret. Specifically, they assume that each input is sampled from a distribution whose concentration is upper bounded by $\frac{1}{\sigma}$ times the uniform distribution. This framework is known as \emph{smoothed analysis}, originally due to \citet{spielman2004smoothed}.\looseness=-1

\textbf{Multiplicative objectives.} Although online learning traditionally studies the sum of payoffs, there is some work which aims to maximize the product of payoffs (or equivalently, the sum of logarithms). See, e.g., Chapter 9 of \citet{cesa2006prediction}. However, these regret bounds are still sublinear in $T$, in comparison to our subconstant regret bounds. Also, like most online learning work, those results assume that payoffs are observed on every time step. In contrast, our agent only receives feedback in response to queries (\Cref{fig:compare}) and never observes payoffs. \citet{barman_fairness_2023} studied a multiplicative objective in a multi-armed bandit context, but their objective is the geometric mean of payoffs instead of the product. Interpreted in our context, their regret bounds imply that the \emph{average} chance of catastrophe goes to zero, while we guarantee that the \emph{total} chance of catastrophe goes to zero. This distinction is closely related to the difference between subconstant and sublinear regret.\looseness=-1


\textbf{Active learning and imitation learning.} Our assumption that the agent only receives feedback in response to queries falls under the umbrella of active learning \citep{hanneke2014active}. This contrasts with passive learning, where the agent receives feedback automatically. The way our agent learns from the mentor is also reminiscent of imitation learning \citep{osa_algorithmic_2018}. Although ideas from these areas could be useful in our setting, we are not aware of any results from that literature which account for irreversible costs. 



\section{Model}\label{sec:model}

\textbf{Inputs.} Let $\bbn$ denote the strictly positive integers and let $T \in \bbn$ be the time horizon.  Let $\s = (x_1,x_2,\dots ,x_T) \in \X^T$ be the sequence of inputs. In the fully adversarial setting, each $x_t$ can have arbitrary (possibly randomized) dependence on the events of prior time steps. In the smoothed setting, the adversary only chooses the distribution $\D_t$ from which $x_t$ is sampled. Formally, a distribution $\D$ over $\X$ is $\sigma$-smooth if for any $S \subseteq \X$, $\D(S) \le \frac{1}{\sigma} U(S)$. (In the smoothed setting, we assume that $\X$ supports a uniform distribution $U$.\footnote{For example, it suffices for $\X$ to have finite Lebesgue measure. Note that this does not imply boundedness. Alternatively, $\sigma$-smoothness can be defined with respect to a different distribution; see Definition 1 of \citet{block2022smoothed}.}) If each $x_t$ is sampled from a $\sigma$-smooth $\D_t$, we say that $\s$ is $\sigma$-smooth. The sequence $\bfd = \D_1,\dots,\D_T$ can still be adaptive, i.e., the choice of $\D_t$ can depend on the events of prior time steps.

\textbf{Actions and payoffs.} Let $\Y$ be a finite set of actions. There also exists a special action $\hata$ which corresponds to querying the mentor. For $k \in \bbn$, let $[k] = \{1,2,\dots, k\}$. On each time step $t \in [T]$, the agent must select action $y_t \in \ \Y\cup \{\hata\}$, which generates a payoff. Let $\A = (y_1,\dots,y_T)$. We allow the payoff function to vary between time steps: let $\bfmu = (\mu_1,\dots,\mu_T) \in (\X \times \Y \to [0,1])^T$ be the sequence of payoff functions. Then $\mu_t(x_t, y_t) \in [0,1]$ is the agent's payoff at time $t$. Like $\bfd$, we allow $\bfmu$ to be adaptive. Unless otherwise noted, all expectations are over any randomization in the agent's decisions, any randomization in $\s$, and any randomization in the adaptive choice of $\bfmu$.


\textbf{Asking for help.} The mentor is endowed with a (possibly suboptimal) policy $\pi^m: \X \to \Y$. When action $\hata$ is chosen, the mentor informs the agent of the action $\pi^m(x_t)$ and the agent obtains payoff $\mu_t(x_t, \pi^m(x_t))$. For brevity, let $\mu_t^m(x) = \mu_t(x, \pi^m(x))$. The agent never observes payoffs: the only way to learn about $\bfmu$ is by querying the mentor.

We would like an algorithm which becomes ``self-sufficient" over time: the rate of querying the mentor should go to 0 as $T\to\infty$, or equivalently, the cumulative number of queries should be sublinear in $T$. Formally, let $Q_T(\bfmu, \pi^m) = \{t \in [T]: y_t = \hata\}$ be the random variable denoting the set of time steps with queries. Then we say that the (expected) number of queries is sublinear in $T$ if $\sup_{\bfmu, \pi^m} \E[|Q_T(\bfmu, \pi^m)|] \in o(T)$. In other words, there must exist $g: \bbn \to \bbn$ which does not depend on $\bfmu$ or $\pi^m$ such that $g(T) \in o(T)$ and $\sup_{\bfmu, \pi^m} \E[|Q_T(\bfmu, \pi^m)|] \le g(T)$. For brevity, we will usually write $Q_T = Q_T(\bfmu,\pi^m)$.

\textbf{Local generalization.} We assume that $\bfmu$ and $\pi^m$ satisfy \emph{local generalization}. Informally, if the agent is given an input $x$, taking the mentor action for a similar input $x'$ is almost as good as taking the mentor action for $x$. Formally, we assume $\X \subseteq \bbr^n$ and there exists $L > 0$ such that for all $x,x' \in \X$ and $t \in [T]$, $|\mu_t^m(x) - \mu_t(x, \pi^m(x'))| \le L \norm{x-x'}$, where $||\cdot||$ denotes Euclidean distance. This represents the ability to transfer knowledge between similar inputs:\looseness=-1
\[
\big|\!\!\underbrace{\mu_t(x, \pi^m(x))}_{\text{Taking the right action}} -\! \underbrace{\mu_t(x, \pi^m(x'))}_{\text{Using what you learned in $x'$}}\!\!\!\!\!\!\!\big| \, \le\: \underbrace{L\norm{x-x'}}_{\text{Input similarity}}
\]
This ability seems fundamental to intelligence and is well-understood in psychology (e.g., \citealp{esser_actioneffect_2023}) and education  (e.g., \citealp{hajian_transfer_2019}).  Note that the input space $\X\subseteq \bbr^n$ can be any encoding of the agent’s situation, not just its physical positioning. See \Cref{sec:lg} for further discussion.

All suprema over $\bfmu,\pi^m$ pairs are assumed to be restricted to $\bfmu, \pi^m$ pairs which satisfy local generalization.

\textbf{Regret.} If $\mu_t(x_t, y_t) \in [0,1]$ is the chance of avoiding catastrophe at time $t$ (conditioned on no prior catastrophe), then by the chain rule of probability, $\prod_{t=1}^T \mu_t(x_t, y_t)$ is the agent's overall chance of avoiding catastrophe. For given $\s, \A, \bfmu, \pi^m$, the agent's \emph{multiplicative regret}\footnote{One could also define the multiplicative regret as $R_T' = \prod_{t=1}^T \mu_t^m(x_t) - \prod_{t=1}^T \mu_t(x_t, y_t)$, but our definition is actually stricter: $\lim_{T\to\infty} \rmul \to 0$ implies $\lim_{T\to\infty} R_T'\to 0$, while the reverse is not true. In particular, $\lim_{T\to\infty} R_T' \to 0$ is trivial if $\lim_{T\to\infty} \prod_{t=1}^T \mu_t^m(x_t) \to 0$.} is
\[
\rmul(\s, \A, \bfmu,  \pi^m) = \log \prod_{t=1}^T \mu_t^m(x_t) - \log \prod_{t=1}^T \mu_t(x_t, y_t)
\]
when all payoffs are strictly positive. To handle the case where some payoffs are zero, we assume the existence of $\mum > 0$ such that $\mu_t^m(x_t) \ge \mum$ always. Thus only the agent's payoffs can be zero, so we can safely define $\rmul(\s, \A, \bfmu,  \pi^m)=\infty$ whenever $\mu_t(x_t,y_t) = 0$ for some $t \in [T]$. We write $\rmul = \rmul(\s,\A, \bfmu, \pi^m)$ for brevity.

The assumption of $\mu_t^m(x_t) \ge \mum > 0$ means that the mentor cannot be abysmal. In fact, we argue that high-stakes AI applications should employ a mentor who is almost always safe, i.e., $\mum \approx 1$. If no such mentor exists for some application, perhaps that application should be avoided altogether. \looseness=-1

We also define the agent's \emph{additive regret} as
\[
\rplus(\s, \A, \bfmu,  \pi^m) = \sum_{t=1}^T \mu_t^m(x_t) - \sum_{t=1}^T \mu_t(x_t, y_t)
\]
and similarly write $\rplus = \rplus(\s, \A, \bfmu,  \pi^m)$ for brevity. For both objectives, we desire subconstant worst-case regret: the total (not average) expected regret should go to 0 for any $\bfmu$ and $\pi^m$. Formally, we want $\lim_{T\to\infty} \sup_{\bfmu, \pi^m} \E[\rmul] = 0$ and $\lim_{T\to\infty} \sup_{\bfmu, \pi^m} \E[\rplus] = 0$.

\textbf{VC and Littlestone dimensions.} VC dimension \citep{vapnik_uniform_1971} and Littlestone dimension \citep{littlestone1988learning} are standard measures of learning difficulty which capture the ability of a hypothesis class (in our case, a policy class) to realize arbitrary combinations of labels (in our case, actions). We omit the precise definitions since we only utilize these concepts via existing results. See \citet{shalev-shwartz_understanding_2014} for a comprehensive overview.\looseness=-1

\textbf{Misc.} The diameter of a set $S \subseteq \X$ is defined by $\diam(S) = \max_{x, x' \in S} \norm{x-x'}$. All logarithms and exponents are base $e$ unless otherwise noted. For convenience, we treat $\min_{x \in \emptyset} f(x)$ as $\infty$ for any function $f$.

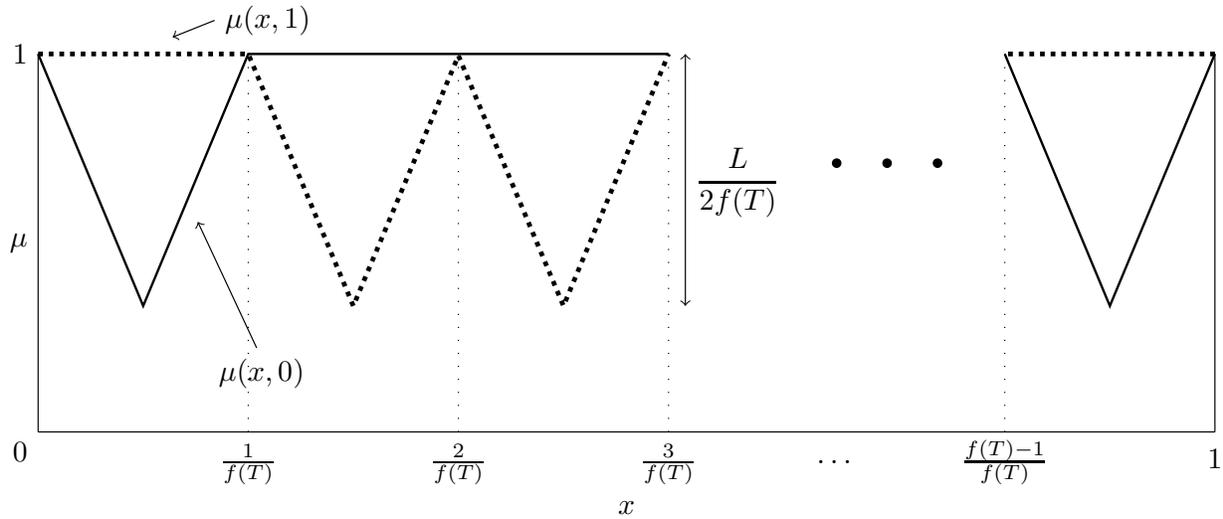
\begin{figure*}[ht]
\centering
\resizebox{0.78\linewidth}{!}{
\begin{tikzpicture}[>={Stealth[scale=1.3]}]
\newcommand\yoffset{2}
\draw (0,\yoffset) -- (14,\yoffset) node[below, midway, label={[label distance=-13mm]{input \large$x$}}] {};
\draw (0,\yoffset) -- (0,6) node[midway, left] {payoff \large$\mu$\ \  };
\draw (14,\yoffset) -- (14,6) node[left] {};

\draw[<-] (1.9,4) -- (2.6,2.8) node[below] {\ $\mu(x,0)$};
\draw[<-] (1.6,6.2) -- (2.1,6.4) node[right] {$\mu(x,1)$};
\draw[<->] (7.7,6) -- (7.7,3) node[midway, right] {$\cfrac{L}{2f(T)}$};

\foreach \s in {2.5,5,7.5,11.5}
\draw[loosely dotted] (\s,\yoffset) -- (\s, 6);

\draw[ultra thick, loosely dotted] (0,6) -- (2.5,6) -- (3.75,3) -- (5,5.98) -- (6.25,3) -- (7.5,6);
\draw[thick] (0,6) -- (1.25,3) -- (2.5,6)-- (7.5,6);
\draw[thick] (14,6) -- (12.75, 3) -- (11.5,6);
\draw[ultra thick, loosely dotted] (14,6) -- (11.5,6);

\draw[fill] (9.5,4.7) circle [radius=0.05];
\draw[fill] (10.1,4.7) circle [radius=0.05];
\draw[fill] (10.7,4.7) circle [radius=0.05];

\node[below,label={[label distance=-8mm]{\large$\frac{1}{f(T)}$}}] at (2.5,\yoffset) {};
\node[below,label={[label distance=-8mm]{\large$\frac{2}{f(T)}$}}] at (5,\yoffset) {};
\node[below,label={[label distance=-8mm]{\large$\frac{3}{f(T)}$}}] at (7.5,\yoffset) {};
\node[below,label={[label distance=-8mm]{\large$\frac{f(T)-1}{f(T)}$}}] at (11.5,\yoffset) {};
\node[below,label={[label distance=-5mm]{\large$\dots$}}] at (9.5,\yoffset) {};
\node[below,label={[label distance=-5.6mm]{1}}] at (14,\yoffset) {};

\node[below left] at (0,\yoffset) {$0$};
\node[left] at (0,6) {$1$};
\end{tikzpicture}
}
\caption{An illustration of the construction we use to prove Theorem~\ref{thm:neg} (not to scale). The horizontal axis indicates the input $x \in [0,1] = \X$ and the vertical axis indicates the payoff $\mu(x,y) \in [0,1]$. The solid line represents $\mu(x,0)$ and the dotted line represents $\mu(x,1)$. In each section, one of the actions has the optimal payoff of 1, and the other action has the worst possible payoff allowed by $L$, reaching a minimum of $1 - \frac{L}{2f(T)}$. Crucially, both actions result in a payoff of 1 at the boundaries between sections: this allows us to ``reset" for the next section. As a result, we can freely toggle the optimal action for each section independently.}
\label{fig:basic}
\end{figure*}

\section{Avoiding catastrophe with sublinear queries is impossible in general}\label{sec:negative}

We first show that in general, any algorithm with sublinear mentor queries has unbounded regret in the worst-case, even when inputs are i.i.d. on $[0,1]$ and $\bfmu$ does not vary over time. The formal proofs are deferred to Appendix~\ref{sec:neg-proof}, but we provide intuition and define the construction here.


\begin{restatable}{theorem}{thmNeg}
\label{thm:neg}
Any algorithm with sublinear queries has unbounded worst-case regret (both multiplicative and additive) as $T\to \infty$. Specifically,
\[
\sup_{\bfmu, \pi^m}\E[\rmul],\, 
\sup_{\bfmu, \pi^m}\E[\rplus] \in \Omega\left(\! L \sqrt{\frac{T}{\sup_{\bfmu,\pi^m}\E[|Q_T|]\! +\! 1}}\right)
\]
\end{restatable}
Intuitively, the regret decreases as the number of queries increases. However, as long as the number of queries remains sublinear in $T$, the regret is unbounded as $T\to\infty$.

We also have the following corollary of \Cref{thm:neg}:

\begin{restatable}{corollary}{corNegProb}
\label{cor:neg-prob}
Even when $\mu_t^m(x) = 1$ for all $t$ and $x$, any algorithm with sublinear queries satisfies
\[
\lim_{T\to\infty}\, \sup_{\bfmu,\pi^m}\, \E\left[\prod_{t=1}^T\mu_t(x_t,y_t)\right] = 0
\]
\end{restatable}
In other words, even if the mentor never causes catastrophe, any algorithm with sublinear queries causes catastrophe with probability 1 as $T\to\infty$ in the worst case.

\subsection{Intuition}\label{sec:neg-intuition}

We partition $\X$ into equally-sized sections that are ``independent" in the sense that querying an input in section $i$ provides no information about section $j$. There will be $f(T)$ sections, where $f$ is a function that we will choose. If $|Q_T| \in o(f(T))$, most of these sections will never contain a query. When the agent sees an input in a section not containing a query, it essentially must guess, meaning it will be wrong about half the time. We then choose a payoff function (which is the same for all time steps) which makes the wrong guesses as costly as possible, subject to the local generalization constraint. Figure~\ref{fig:basic} fleshes out this idea.


The choice of $f$ is crucial. One idea is $f(T) = T$. If the agent is wrong about half the time, and the average payoff for wrong actions is $1-\frac{L}{4T}$, we can estimate the regret as\looseness=-1
\begin{align*}
\rmul\, =&\ \log \prod_{t=1}^T \mu_t^m(x_t) - \log \prod_{t=1}^T \mu_t(x_t, y_t)\\
\approx&\ \log 1 - \log \left(1-\frac{L}{4T}\right)^{T/2}\\
=&\ - \frac{T}{2} \log \left(1-\frac{L}{4T}\right)\\
\approx&\ \frac{T}{2} \cdot \frac{L}{4T}\\
=&\ \frac{L}{8}
\end{align*}
Thus $f(T) = T$ can at best give us a constant lower bound on regret. Instead, we choose $f$ such that $|Q_T| \in o(f(T))$ and $f(T) \in o(T)$. Specifically, we choose $f(T) = \max(\sqrt{\sup_{\bfmu,\pi^m}\E[|Q_T|] T},1)$. Most sections still will not contain a query, so the agent is still wrong about half the time, but the payoff for wrong actions is worse. Then
\begin{align*}
\rmul\, \approx&\ \log 1 - \log \left(1-\frac{L}{4f(T)}\right)^{T/2}\\
\approx&\  \frac{LT}{8f(T)}\\
\in&\  \Omega\left(L \sqrt{\frac{T}{\sup_{\bfmu,\pi^m}\E[|Q_T|] + 1}}\right)
\end{align*}
which produces the bound in \Cref{thm:neg}.

\textbf{VC dimension.} The class of mentor policies in our construction has VC dimension $f(T)$; across all possible values of $T$, this implies infinite VC (and Littlestone) dimension. We know that this is necessary given our positive results.

\subsection{Formal definition of construction}\label{sec:neg-construction}

Let $\X = [0,1]$ and $\D_t = U(\X)$ for each $t \in [T]$, where $U(\X)$ is the uniform distribution on $\X$. Assume that $L \le 1$; this will simplify the math and only makes the problem easier for the agent. We define a family of payoff functions parameterized by a function $f: \bbn \to \bbn$ and a bit string $\x  = (a_1, a_2,\dots, a_{f(T)}) \in \{0,1\}^{f(T)}$. The bit $a_j$ will denote the optimal action in section $j$. Note that $f(T) \ge 1$.

For each $j\in [f(T)]$, we refer to $X_j = \left[\frac{j-1}{f(T)}, \frac{j}{f(T)}\right]$ as the $j$th section. Let $m_j = \frac{j-0.5}{f(T)}$ be the midpoint of $X_j$. Assume that each $x_t$ belongs to exactly one $X_j$ (this happens with probability 1, so this assumption does not affect the expected regret). Let $j(x)$ denote the index of the section containing input $x$.  Then $\mufx$ is defined by

\resizebox{\linewidth}{!}{%
$\displaystyle
\mufx(x,y) = 
\begin{cases}
1 & \textnormal{ if } y = a_{j(x)}\\
1 - L \left(\cfrac{1}{2f(T)} - |m_{j(x)} - x|\right)  & \textnormal{ if } y \ne a_{j(x)}\\
\end{cases}
$
}


We use this payoff function for all time steps: $\mu_t = \mufx$ for all $t \in [T]$. Let $\pi^m$ be any optimal policy for $\mufx$. Note 
that there is a unique optimal action for each $x_t$, since each $x_t$ belongs to exactly one $X_j$; formally, $\pi^m(x_t) = a_{j(x_t)}$. 

For any $\x \in \{0,1\}^{f(T)}$, $\mufx$ is piecewise linear (trivially) and continuous (because both actions have payoff 1 on the boundary between sections). Since the slope of each piece is in $\{-L, 0, L\}$, $\mufx$ is Lipschitz continuous. Thus by \Cref{prop:lipschitz}, $\pi^m$ satisfies local generalization.


\section{Avoiding catastrophe given finite VC or Littlestone dimension}\label{sec:pos-nd}

\Cref{thm:neg} shows that avoiding catastrophe is impossible in general. What if we restrict ourselves to settings where standard online learning is possible? Specifically, we assume that $\pi^m$ belongs to a policy class $\Pi$ where either (1) $\Pi$ has finite VC dimension $d$ and $\s$ is $\sigma$-smooth or (2) $\Pi$ has finite Littlestone dimension $d$.\footnote{Recall from \Cref{sec:online} that standard online learning becomes tractable under either of these assumptions.} This section presents a simple algorithm which guarantees subconstant regret (both multiplicative and additive) and sublinear queries under either of those assumptions. Formal proofs are deferred to \Cref{sec:nd-proof} but we provide intuition and a proof sketch here.\looseness=-1

\subsection{Intuition behind the algorithm}\label{sec:nd-intuition}

\Cref{alg:nd} has two simple components: (1) run a modified version of the Hedge algorithm for online learning, but (2) ask for help for unfamiliar inputs (specifically, when the input is very different from any queried input with the same action under the proposed policy). Hedge ensures that the number of mistakes (i.e., the number of time steps where the agent's action doesn't match the mentor's) is small, and asking for help for unfamiliar inputs ensures that when we do make a mistake, the cost isn't too high. This algorithmic structure seems quite natural: mostly follow a baseline strategy, but ask for help when out-of-distribution.\looseness=-1


\textbf{Hedge.} Hedge \citep{freund1997decision} is a standard online learning algorithm which ensures sublinear regret when the number of hypotheses (in our case, the number of policies in $\Pi$) is finite.\footnote{Chapter 5 of \citet{slivkins2019introduction} and Chapter 21 of \citet{shalev-shwartz_understanding_2014} give modern introductions to Hedge.} We would prefer not to assume that $\Pi$ is finite. Luckily, any policy in $\Pi$ can be approximated within $\ep$ when either (1) $\Pi$ has finite VC dimension and $\s$ is $\sigma$-smooth or (2) $\Pi$ has finite Littlestone dimension. Thus we can run Hedge on this approximative policy class instead.

One other modification is necessary. In standard online learning, losses are observed on every time step, but our agent only receives feedback in response to queries. To handle this, we modify Hedge to only perform updates on time steps with queries and to issue a query with probability $p$ on each time step. Continuing our lucky streak, \citet{russo2024online} analyze exactly this modification of Hedge.


\subsection{Local generalization}\label{sec:lg}

Local generalization is vital: this is what allows us to detect when an input is unfamiliar. Crucially, our algorithm does not need to know how inputs are encoded in $\bbr^n$ and does not need to know $L$: it only needs to be able to compute the nearest-neighbor distance $\min_{(x,y) \in S: y = \pi_t(x_t)} \norm{x_t - x}$. Thus we only need to assume that there exists \emph{some} encoding satisfying local generalization. \looseness=-1

To elaborate, recall the example that a 3 mm spot and a 3.1 mm spot on X-rays likely have similar risk levels (assuming similar density, location, etc.). If the risk level abruptly increases for any spot over 3 mm, then local generalization may not hold for a naive encoding which treats size as a single dimension. However, a more nuanced encoding would recognize that these two situations -- a 3 mm vs 3.1 mm spot -- are in fact \emph{not} similar. Constructing a suitable encoding may be challenging, but we do \emph{not} require the agent to have explicit access to such an encoding: the agent only needs a nearest-neighbor distance oracle.

\begin{algorithm}[t]
\centering
\begin{algorithmic}

\STATE Inputs: $T \in \bbn,\: \ep \in \bbrspos,\: d \in \bbn,\,$ policy class $\Pi$
\IF{$\Pi$ has VC dimension $d$}
    \STATE $\tilpi \gets$ any smooth $\ep$-cover of $\Pi$ of size at most $(41/\ep)^d$ \ (see \Cref{def:smooth-cover})
\ELSE
    \STATE $\tilpi \gets$ any adversarial cover of $\Pi$ of size at most $(eT/d)^d$ \ (see \Cref{def:adv-cover})
\ENDIF
\STATE $S \gets \emptyset$
\STATE $w(\pi) \gets 1$ for all $\pi \in \tilpi$
\STATE $p \gets 1/\sqrt{\ep T}$
\STATE $\eta \gets \max\big(\sqrt{\frac{p\log |\tilpi|}{2T}},\: \frac{p^2}{\sqrt{2}} \big)$
\FOR{$t$ \textbf{from} $1$ \textbf{to} $T$} \STATE \textit{Run one step of Hedge, which selects policy $\pi_t$}
        \STATE with probability $p:$ $\texttt{hedgeQuery} \gets \texttt{true}$
        \STATE with probability $1-p:$ $\texttt{hedgeQuery} \gets \texttt{false}$
        \IF{\texttt{hedgeQuery}\footnotemark}
                \STATE Query mentor and observe $\pi^m(x_t)$
                \STATE $\ell(t,\pi) \gets \bfone(\pi(x_t) \ne \pi^m(x_t))$ for all $\pi \in \tilpi$
                \STATE $\ell^* \gets \min_{\pi \in \tilpi} \ell(t,\pi)$
                \STATE $w(\pi) \gets w(\pi)\cdot \exp(-\eta(\ell(t,\pi) - \ell^*))$ for all $\pi \in \tilpi$
                \STATE $\pi_t \gets \argmin_{\pi \in \tilpi} \ell(t, \pi)$
        \ELSE 
            \STATE $P(\pi) \gets w(\pi) / \sum_{\pi' \in \tilpi} w(\pi')$ for all $ \pi \in \tilpi$
            \STATE Sample $\pi_t \sim P$ 
        \ENDIF
            \IF {$\min_{(x,y) \in S: y = \pi_t(x_t)}  \norm{x_t - x} > \ep^{1/n}$\looseness=-1} \STATE \textit{Ask for help if out-of-distribution}
                        \STATE Query mentor (if not already queried this round) and observe $\pi^m(x_t)$
                        \STATE $S \gets S \cup\{(x_t, \pi^m(x_t))\}$   
                \ELSE \STATE \textit{Otherwise, follow Hedge's chosen policy}
                                        \STATE Take action $\pi_t(x_t)$
                \ENDIF
\ENDFOR
\end{algorithmic}
\caption{successfully avoids catastrophe assuming finite VC or Littlestone dimension.}
\label{alg:nd}
\end{algorithm}

Conceptually, the algorithm only needs to be able to detect when an input is unfamiliar. While this task remains far from trivial, we argue that it is more tractable than fully constructing a suitable encoding. See \Cref{sec:conclusion} for a discussion of potential future work on this topic.

We note that these encoding-related questions apply similarly to the more standard assumption of Lipschitz continuity. In fact, Lipschitz continuity implies local generalization when the mentor is optimal (\Cref{prop:lipschitz}). We also mention that without local generalization, avoiding catastrophe is impossible even when the mentor policy class has finite VC dimension and $\s$ is $\sigma$-smooth (\Cref{thm:no-lg}).

\subsection{Main result}

For simplicity, here we only state our results for $\Y =\{0,1\}$; \Cref{sec:multi} extends our result to many actions using the standard ``one versus rest'' reduction. We first prove regret and query bounds parametrized by $\ep$:

\footnotetext{The reader may notice that we do not update $S$ in this case. This is simply because those updates are not necessary for the desired bounds and omitting these updates simplifies the analysis.}

\begin{restatable}{theorem}{thmNDep}\label{thm:pos-nd-ep}
Let $\Y = \{0,1\}$. Assume $\pi^m \in \Pi$ where either (1) $\Pi$ has finite VC dimension $d$ and $\s$ is $\sigma$-smooth, or (2) $\Pi$ has finite Littlestone dimension $d$. Then for any $T \in \bbn$ and $\ep \in \left[\mfrac{1}{T}, \left(\mfrac{\mum}{2L}\right)^n\right]$,  \Cref{alg:nd} satisfies\looseness=-1
\begin{align*}
\E \left[\rmul\right] \in&\ O\left(\frac{d L}{\sigma \mum} T \ep^{1+1/n} \log (T + 1/\ep) \right)\\
\E \left[\rplus\right] \in&\ O\left(\frac{d L}{\sigma} T \ep^{1+1/n} \log (T + 1/\ep) \right)\\
\E[|Q_T|] \in&\ O\left(\! \sqrt{\frac{T}{\ep}} +\frac{d}{\sigma} T \ep \log (T\! +\! 1/\ep)  + \frac{\E[\diam(\s)^n]}{\ep}\right)
\end{align*}
\end{restatable}

In Case 1, the expectation is over the randomness of both $\s$ and the algorithm, while in Case 2, the expectation is over only the randomness of the algorithm. The bounds clearly have no dependence on $\sigma$ in Case 2, but we include $\sigma$ anyway to avoid writing two separate sets of bounds.

 
To obtain subconstant regret and sublinear queries, we can choose $\ep = T^\frac{-2n}{2n+1}$. This also satisfies $2 L\ep^{1/n} \le \mum$ for large enough $T$.

\begin{restatable}{theorem}{thmND}
\label{thm:pos-nd}
Let $\Y = \{0,1\}$. Assume $\pi^m \in \Pi$ where either (1) $\Pi$ has finite VC dimension $d$ and $\s$ is $\sigma$-smooth or (2) $\Pi$ has finite Littlestone dimension $d$. Then for any $T \in \bbn$, \Cref{alg:nd} with $\ep = T^\frac{-2n}{2n+1}$ satisfies
\begin{align*}
\E \left[\rmul\right] \in&\ O\left(\frac{d L}{\sigma \mum} T^{\frac{-1}{2n+1}} \log T\right)\\
\E \left[\rplus\right] \in&\ O\left(\frac{d L}{\sigma} T^{\frac{-1}{2n+1}} \log T\right)\\
\E[|Q_T|] \in&\ O\left(T^\frac{4n+1}{4n+2}\left(\frac{d}{\sigma} \log T + \E[\diam(\s)^n]\right)\right) 
\end{align*}
\end{restatable}


Before proceeding to the proof sketch, we highlight some advantages of our algorithm.

\textbf{Limited knowledge required.} Our algorithm needs to know $\Pi$, as is standard. However, the algorithm does not need to know $\sigma$ (in the smooth case) or $L$. Although \Cref{alg:nd} as written does require $T$ as an input, it can be converted into an infinite horizon/anytime algorithm via the standard ``doubling trick'' (see, e.g., \citealp{slivkins2019introduction}).

\textbf{Unbounded environment.} Our algorithm can handle an unbounded input space: the number of queries simply scales with the maximum distance between observed inputs in the form of $\E[\diam(\s)^n]$.

\textbf{Simultaneous bounds for all $\bfmu$.} Recall that the agent never observes payoffs and only learns from mentor queries. This means that the agent's behavior does not depend on $\bfmu$ at all. In other words, the distribution of $(\s,\A)$ depends on $\pi^m$ but not $\bfmu$. Consequently, a given $\pi^m$ induces a \emph{single} distribution $(\s,\A)$ which satisfies the bounds in \Cref{thm:pos-nd} \emph{simultaneously} for all $\bfmu$ satisfying local generalization.

\subsection{Proof sketch}

The formal proof of \Cref{thm:pos-nd-ep} can be found in Appendix~\ref{sec:nd-proof}, but we outline the key elements here. The regret analysis consists of two ingredients: analyzing the Hedge component and analyzing the ``ask for help when out-of-distribution'' component. The former will bound the number of mistakes made by the algorithm (i.e., the number of time steps where the agent's action doesn't match the mentor's), and the latter will bound the cost of any single mistake. We must also show that the latter does not result in excessively many queries, which we do via a novel packing argument.


We begin by formalizing two notions of approximating a policy class:

\begin{definition}
\label{def:smooth-cover}
Let $U$ be the uniform distribution over $\X$. For $\ep > 0$, a policy class $\tilpi$ is a \emph{smooth $\ep$-cover} of a policy class $\Pi$ if for every $\pi \in \Pi$, there exists $\stilpi \in \tilpi$ such that $\Pr_{x\sim U} [\pi(x) \ne \stilpi(x)] \le \ep$.
\end{definition}

\begin{definition}
\label{def:adv-cover}
A policy class $\tilpi$ is an \emph{adversarial cover} of a policy class $\Pi$ if for every $\s \in \X^T$ and $\pi \in \Pi$, there exists $\stilpi \in \tilpi$ such that $\pi(x_t) = \stilpi(x_t)$ for all $t \in [T]$.
\end{definition}

An adversarial cover is a perfect cover by definition. The idea of a smooth $\ep$-cover is that if the probability of disagreement over the uniform distribution is small, then the probability of disagreement over a $\sigma$-smooth distribution cannot be too much larger.
\begin{lemma}
\label{lem:smooth-concentrate}
Let $\tilpi$ be a smooth $\ep$-cover of $\Pi$ and let $\D$ be a $\sigma$-smooth distribution. Then for any $\pi \in \Pi$, there exists $\stilpi \in \tilpi$ such that $\Pr_{x \sim \D}[\pi(x) \ne \stilpi(x)] \le \ep/\sigma$.
\end{lemma}

\begin{proof}
Define $S(\stilpi) = \{x \in \X: \pi(x) \ne \stilpi(x)\}$. By the definition of a smooth $\ep$-cover, there exists $\stilpi \in \tilpi$ such that $\Pr_{x\sim U}[x \in S(\stilpi)] \le \ep$. Since $\D$ is $\sigma$-smooth, $\Pr_{x \sim \D}[\pi(x) \ne \stilpi(x)] =\Pr_{x\sim \D}[x \in S(\stilpi)] \le \\\Pr_{x\sim U}[x \in S(\stilpi)]/\sigma \le \ep/\sigma$, as claimed.
\end{proof}

The existence of small covers is crucial:


\begin{lemma}[Lemma 7.3.2 in \citet{haghtalab2018foundation}\footnote{See also \citet{haussler_generalization_1995} or Lemma 13.6 in \citet{boucheron2013concentration} for variants of this lemma.}]
\label{lem:smooth-cover}
For all $\ep > 0$, any policy class of VC dimension $d$ admits a smooth $\ep$-cover of size at most $(41 /\ep)^d$.
\end{lemma}

\begin{lemma}[Lemmas 21.13 and A.5 in \citet{shalev-shwartz_understanding_2014}]
\label{lem:adv-cover}
Any policy class of Littlestone dimension $d$ admits an adversarial cover of size at most $(eT/d)^d$.
\end{lemma}




We will run a variant of Hedge on $\tilpi$. The vanilla Hedge algorithm operates in the standard online learning model where on each time step, the agent selects a policy (or more generally, a hypothesis), and observes the \emph{loss} of every policy. In general the loss function can depend arbitrarily on the time step, the policy, and prior events, but we will only use the indicator loss function $\ell(t, \pi) = \bfone(\pi(x_t) \ne \pi^m(x_t))$. Crucially, whenever we query and learn $\pi^m(x_t)$, we can compute $\ell(t,\pi)$ for every $\pi \in \tilpi$.

We cannot afford to query on every time step, however. Recently, \citet{russo2024online} analyzed a variant of Hedge where losses are observed only in response to queries, which they call ``label-efficient feedback". They proved a regret bound when a query is issued on each time step with fixed probability $p$. \Cref{lem:russo} restates their result in a form that is more convenient for us. See Appendices~\ref{sec:russo} and \ref{sec:adaptive} for details on our usage of results from \citet{russo2024online}. Full pseudocode for \textsc{HedgeWithQueries} can also be found in the appendix (\Cref{alg:hedge}).

\begin{restatable}[Lemma 3.5 in \citealp{russo2024online}]{lemma}{lemRusso}\label{lem:russo}
Assume $\tilpi$ is finite. Then for any loss function $\ell: [T] \times \tilpi \to [0,1]$ and query probability $p >0$, \textsc{HedgeWithQueries} enjoys the regret bound 
\[
\sum_{t=1}^T \E[\ell(t, \pi_t)] - \min_{\stilpi \in \tilpi} \sum_{t=1}^T \E[\ell(t, \pi)] \le \frac{\log |\tilpi| }{p^2}
\]
where $\pi_t$ is the policy chosen at time $t$ and the expectation is over the randomness of the algorithm.
\end{restatable}

We apply \Cref{lem:russo} with $\ell(t, \pi) = \bfone(\pi(x_t) \ne \pi^m(x_t))$ and combine this with Lemmas \ref{lem:smooth-cover} and \ref{lem:smooth-concentrate} (in the $\sigma$-smooth case) and with \Cref{lem:adv-cover} (in the adversarial case). This  yields a $O\left(\frac{d}{\sigma} T \ep \log(1/\ep) \log T\right)$ bound on the number of mistakes made by \Cref{alg:nd} (\Cref{lem:nd-num-errors}).


The other key ingredient of the proof is analyzing the ``ask for help when out-of-distribution'' component. Combined with the local generalization assumption, this allows us to fairly easily bound the cost of a single mistake (\Cref{lem:nd-lipschitz-payoff}). The trickier part is bounding the number of resulting queries. It is tempting to claim that the inputs queried in the out-of-distribution case must all be separated by at least $\ep^{1/n}$ and thus form an $\ep^{1/n}$-packing, but this is actually false.

Instead, we bound the number of data points (i.e., queries) needed to cover a set \emph{with respect to the realized actions of the algorithm}  (\Cref{lem:pos-nd-queries}). This contrasts with vanilla packing arguments which consider all data points in aggregate. The key to our analysis is that the number of mistakes made by the algorithm -- which we already bounded in \Cref{lem:nd-num-errors} -- gives us crucial information about how data points are distributed with respect to the actions of the algorithm. Our technique might be useful in other contexts where a more refined packing argument is needed and a bound on the number of mistakes already exists.






\section{Conclusion and future work}\label{sec:conclusion}

In this paper, we proposed a model of avoiding catastrophe in online learning. We showed that achieving subconstant regret in our problem (with the help of a mentor and local generalization) is no harder than achieving sublinear regret in standard online learning.

\textbf{Remaining technical questions.} First, we have not resolved whether our problem is tractable for finite VC dimension and fully adversarial inputs (although \Cref{sec:positive-1d} shows that the problem is tractable for at least some classes with finite VC but infinite Littlestone dimension). Second, the time complexity of \Cref{alg:nd} currently stands at a hefty $\Omega(|\tilpi|)$ per time step plus the time to compute $\tilpi$. In the standard online learning setting, \citet{block2022smoothed} and \citet{haghtalab2022oracle} show how to replace discretization approaches like ours with \emph{oracle-efficient} approaches, where a small number of calls to an optimization oracle are made per round.  We are optimistic about leveraging such techniques to obtain efficient algorithms in our setting. 


\textbf{Local generalization.} Our algorithm crucially relies on the ability to detect when an input is unfamiliar, i.e., differs significantly from prior observations in a metric space which satisfies local generalization. Without this ability, the practicality of our algorithm would be fundamentally limited. One option is to use out-of-distribution (OOD) detection, which is conceptually similar and well-studied (see \citealp{yang2024generalized} for a survey). However, it is an open question whether standard OOD detection methods are measuring distance in a metric space which satisfies local generalization.

We are also interested in alternatives to local generalization. \Cref{thm:no-lg} shows that our positive result breaks down if local generalization is removed, so some sort of assumption is necessary. One possible alternative is Bayesian inference. We intentionally avoided Bayesian approaches in this paper due to tractability concerns, but it seems premature to abandon those ideas entirely.


\textbf{MDPs.} Finally, we are excited to apply the ideas in this paper to Markov Decision Processes (MDPs):
specifically, MDPs where some actions are irreversible (``non-communicating'') and the agent only gets one attempt (``single-episode''). In such MDPs, the agent must not only avoid catastrophe but also obtain high reward. As discussed in \Cref{sec:related}, very little theory exists for RL in non-communicating single-episode MDPs. Can an agent learn near-optimal behavior in high-stakes environments while becoming self-sufficient over time? Formally, we pose the following open problem:

\begin{center}
\textit{Is there an algorithm for non-communicating single-episode undiscounted MDPs which ensures that both the regret and the number of mentor queries are sublinear in $T$?}
\end{center}


\section*{Impact statement}

As AI systems become increasingly powerful, we believe that the safety guarantees of such systems should become commensurately robust. Irreversible costs are especially worrisome, and we hope that our work plays a small part in mitigating such risks. We do not believe that our work has any concrete potential risks that should be highlighted here.

\newpage

\section*{Author contributions}

B. Plaut conceived the project. B. Plaut designed the mathematical model with feedback from H. Zhu and S. Russell. B. Plaut proved all of the results. H. Zhu participated in proof brainstorming and designed counterexamples for several early conjectures. B. Plaut wrote the paper, with feedback from H. Zhu and S. Russell. S. Russell supervised the project and secured funding.

\section*{Acknowledgements}

This work was supported in part by a gift from Open Philanthropy to the Center for Human-Compatible AI (CHAI) at UC Berkeley. This paper also benefited from discussions with many other researchers. We would like to especially thank (in alphabetical order) Aly Lidayan, Bhaskar Mishra, Cameron Allen, Juan Li\'{e}vano-Karim, Matteo Russo, Michael Cohen, Nika Haghtalab, and Scott Emmons. We would also like to thank our anonymous reviewers for helpful feedback.

\bibliography{refs}

\begin{thebibliography}{51}
\providecommand{\natexlab}[1]{#1}
\providecommand{\url}[1]{\texttt{#1}}
\expandafter\ifx\csname urlstyle\endcsname\relax
  \providecommand{\doi}[1]{doi: #1}\else
  \providecommand{\doi}{doi: \begingroup \urlstyle{rm}\Url}\fi

\bibitem[Abaimov \& Martellini(2020)Abaimov and Martellini]{Abaimov2020}
Abaimov, S. and Martellini, M.
\newblock \emph{Artificial Intelligence in Autonomous Weapon Systems}, pp.\  141--177.
\newblock Springer International Publishing, Cham, 2020.

\bibitem[Altman(2021)]{altman2021constrained}
Altman, E.
\newblock \emph{Constrained Markov Decision Processes}.
\newblock Routledge, 2021.

\bibitem[Azar et~al.(2017)Azar, Osband, and Munos]{azar_minimax_2017}
Azar, M.~G., Osband, I., and Munos, R.
\newblock Minimax {Regret} {Bounds} for {Reinforcement} {Learning}.
\newblock In \emph{Proceedings of the 34th {International} {Conference} on {Machine} {Learning}}, pp.\  263--272. PMLR, July 2017.
\newblock ISSN: 2640-3498.

\bibitem[Barman et~al.(2023)Barman, Khan, Maiti, and Sawarni]{barman_fairness_2023}
Barman, S., Khan, A., Maiti, A., and Sawarni, A.
\newblock Fairness and welfare quantification for regret in multi-armed bandits.
\newblock In \emph{Proceedings of the {Thirty}-{Seventh} {Conference} on {Artificial} {Intelligence} (AAAI 2023)}, 2023.

\bibitem[Bengio et~al.(2024)Bengio, Hinton, Yao, Song, Abbeel, Darrell, Harari, Zhang, Xue, Shalev-Shwartz, et~al.]{bengio2024managing}
Bengio, Y., Hinton, G., Yao, A., Song, D., Abbeel, P., Darrell, T., Harari, Y.~N., Zhang, Y.-Q., Xue, L., Shalev-Shwartz, S., et~al.
\newblock Managing extreme {AI} risks amid rapid progress.
\newblock \emph{Science}, 384\penalty0 (6698):\penalty0 842--845, 2024.

\bibitem[Betancourt(2018)]{betancourt_conceptual_2018}
Betancourt, M.
\newblock A {Conceptual} {Introduction} to {Hamiltonian} {Monte} {Carlo}, July 2018.
\newblock arXiv:1701.02434 [stat].

\bibitem[Block et~al.(2022)Block, Dagan, Golowich, and Rakhlin]{block2022smoothed}
Block, A., Dagan, Y., Golowich, N., and Rakhlin, A.
\newblock Smoothed online learning is as easy as statistical learning.
\newblock In \emph{Conference on Learning Theory}, pp.\  1716--1786. PMLR, 2022.

\bibitem[Boucheron et~al.(2013)Boucheron, Lugosi, and Massart]{boucheron2013concentration}
Boucheron, S., Lugosi, G., and Massart, P.
\newblock \emph{{Concentration Inequalities: A Nonasymptotic Theory of Independence}}.
\newblock Oxford University Press, 02 2013.

\bibitem[Cesa-Bianchi \& Lugosi(2006)Cesa-Bianchi and Lugosi]{cesa2006prediction}
Cesa-Bianchi, N. and Lugosi, G.
\newblock \emph{Prediction, learning, and games}.
\newblock Cambridge university press, 2006.

\bibitem[Cohen \& Hutter(2020)Cohen and Hutter]{cohen_pessimism_2020}
Cohen, M.~K. and Hutter, M.
\newblock Pessimism {About} {Unknown} {Unknowns} {Inspires} {Conservatism}.
\newblock In \emph{Proceedings of {Thirty} {Third} {Conference} on {Learning} {Theory}}, pp.\  1344--1373. PMLR, July 2020.
\newblock ISSN: 2640-3498.

\bibitem[Cohen et~al.(2021)Cohen, Catt, and Hutter]{cohen_curiosity_2021}
Cohen, M.~K., Catt, E., and Hutter, M.
\newblock Curiosity {Killed} or {Incapacitated} the {Cat} and the {Asymptotically} {Optimal} {Agent}.
\newblock \emph{IEEE Journal on Selected Areas in Information Theory}, 2\penalty0 (2):\penalty0 665--677, June 2021.
\newblock Conference Name: IEEE Journal on Selected Areas in Information Theory.

\bibitem[Critch \& Russell(2023)Critch and Russell]{critch2023tasra}
Critch, A. and Russell, S.
\newblock {TASRA}: a taxonomy and analysis of societal-scale risks from {AI}.
\newblock \emph{arXiv preprint arXiv:2306.06924}, 2023.

\bibitem[Esser et~al.(2023)Esser, Haider, Lustig, Tanaka, and Tanaka]{esser_actioneffect_2023}
Esser, S., Haider, H., Lustig, C., Tanaka, T., and Tanaka, K.
\newblock Action–effect knowledge transfers to similar effect stimuli.
\newblock \emph{Psychological Research}, 87\penalty0 (7):\penalty0 2249--2258, October 2023.

\bibitem[Freund \& Schapire(1997)Freund and Schapire]{freund1997decision}
Freund, Y. and Schapire, R.~E.
\newblock A decision-theoretic generalization of on-line learning and an application to boosting.
\newblock \emph{Journal of computer and system sciences}, 55\penalty0 (1):\penalty0 119--139, 1997.

\bibitem[Grinsztajn et~al.(2021)Grinsztajn, Ferret, Pietquin, Preux, and Geist]{grinsztajn_there_2021}
Grinsztajn, N., Ferret, J., Pietquin, O., Preux, P., and Geist, M.
\newblock There {Is} {No} {Turning} {Back}: {A} {Self}-{Supervised} {Approach} for {Reversibility}-{Aware} {Reinforcement} {Learning}.
\newblock In \emph{Advances in {Neural} {Information} {Processing} {Systems}}, volume~34, pp.\  1898--1911. Curran Associates, Inc., 2021.

\bibitem[Gu et~al.(2024)Gu, Yang, Du, Chen, Walter, Wang, and Knoll]{gu_review_2023}
Gu, S., Yang, L., Du, Y., Chen, G., Walter, F., Wang, J., and Knoll, A.
\newblock A review of safe reinforcement learning: Methods, theories, and applications.
\newblock \emph{IEEE Transactions on Pattern Analysis and Machine Intelligence}, 46\penalty0 (12):\penalty0 11216--11235, 2024.

\bibitem[Guembe et~al.(2022)Guembe, Azeta, Misra, Osamor, Fernandez-Sanz, and Pospelova]{guembe_emerging_2022}
Guembe, B., Azeta, A., Misra, S., Osamor, V.~C., Fernandez-Sanz, L., and Pospelova, V.
\newblock The {Emerging} {Threat} of {AI}-driven {Cyber} {Attacks}: {A} {Review}.
\newblock \emph{Applied Artificial Intelligence}, 36\penalty0 (1), December 2022.

\bibitem[Hadfield-Menell et~al.(2017)Hadfield-Menell, Milli, Abbeel, Russell, and Dragan]{hadfield-menell_inverse_2017}
Hadfield-Menell, D., Milli, S., Abbeel, P., Russell, S., and Dragan, A.~D.
\newblock Inverse reward design.
\newblock In \emph{Proceedings of the 31st {International} {Conference} on {Neural} {Information} {Processing} {Systems}}, {NIPS}'17, pp.\  6768--6777, Red Hook, NY, USA, December 2017. Curran Associates Inc.

\bibitem[Haghtalab(2018)]{haghtalab2018foundation}
Haghtalab, N.
\newblock \emph{Foundation of Machine Learning, by the People, for the People}.
\newblock PhD thesis, Microsoft Research, 2018.

\bibitem[Haghtalab et~al.(2022)Haghtalab, Han, Shetty, and Yang]{haghtalab2022oracle}
Haghtalab, N., Han, Y., Shetty, A., and Yang, K.
\newblock Oracle-efficient online learning for smoothed adversaries.
\newblock \emph{Advances in Neural Information Processing Systems}, 35:\penalty0 4072--4084, 2022.

\bibitem[Haghtalab et~al.(2024)Haghtalab, Roughgarden, and Shetty]{haghtalab2024smoothed}
Haghtalab, N., Roughgarden, T., and Shetty, A.
\newblock Smoothed analysis with adaptive adversaries.
\newblock \emph{Journal of the ACM}, 71\penalty0 (3):\penalty0 1--34, 2024.

\bibitem[Hajian(2019)]{hajian_transfer_2019}
Hajian, S.
\newblock Transfer of {Learning} and {Teaching}: {A} {Review} of {Transfer} {Theories} and {Effective} {Instructional} {Practices}.
\newblock \emph{IAFOR Journal of Education}, 7\penalty0 (1):\penalty0 93--111, 2019.
\newblock Publisher: International Academic Forum ERIC Number: EJ1217940.

\bibitem[Hanneke(2014)]{hanneke2014active}
Hanneke, S.
\newblock \emph{Theory of Disagreement-Based Active Learning}, volume~7.
\newblock Now Publishers Inc., Hanover, MA, USA, June 2014.

\bibitem[Haussler \& Long(1995)Haussler and Long]{haussler_generalization_1995}
Haussler, D. and Long, P.~M.
\newblock A generalization of {Sauer}'s lemma.
\newblock \emph{Journal of Combinatorial Theory, Series A}, 71\penalty0 (2):\penalty0 219--240, August 1995.

\bibitem[Hendrycks et~al.(2023)Hendrycks, Mazeika, and Woodside]{hendrycks2023overviewcatastrophicairisks}
Hendrycks, D., Mazeika, M., and Woodside, T.
\newblock An overview of catastrophic {AI} risks.
\newblock \emph{arXiv preprint arXiv:2306.12001}, 2023.

\bibitem[Jaksch et~al.(2010)Jaksch, Ortner, and Auer]{jaksch_near-optimal_2010}
Jaksch, T., Ortner, R., and Auer, P.
\newblock Near-optimal {Regret} {Bounds} for {Reinforcement} {Learning}.
\newblock \emph{Journal of Machine Learning Research}, 11\penalty0 (51):\penalty0 1563--1600, 2010.

\bibitem[Jung(1901)]{Jung1901}
Jung, H.
\newblock Ueber die kleinste kugel, die eine räumliche figur einschliesst.
\newblock \emph{Journal für die reine und angewandte Mathematik}, 123:\penalty0 241--257, 1901.

\bibitem[Kohli \& Chadha(2020)Kohli and Chadha]{kohli2020enabling}
Kohli, P. and Chadha, A.
\newblock Enabling pedestrian safety using computer vision techniques: A case study of the 2018 {Uber} {Inc.} self-driving car crash.
\newblock In \emph{Advances in Information and Communication: Proceedings of the 2019 Future of Information and Communication Conference (FICC), Volume 1}, pp.\  261--279. Springer, 2020.

\bibitem[Kosoy(2019)]{kosoy_delegative_2019}
Kosoy, V.
\newblock Delegative {Reinforcement} {Learning}: learning to avoid traps with a little help.
\newblock SafeML ICLR 2019 Workshop, July 2019.

\bibitem[Littlestone(1988)]{littlestone1988learning}
Littlestone, N.
\newblock Learning quickly when irrelevant attributes abound: A new linear-threshold algorithm.
\newblock \emph{Machine learning}, 2:\penalty0 285--318, 1988.

\bibitem[Liu et~al.(2021)Liu, Zhou, Kalathil, Kumar, and Tian]{liu2021learning}
Liu, T., Zhou, R., Kalathil, D., Kumar, P., and Tian, C.
\newblock Learning policies with zero or bounded constraint violation for constrained {mdps}.
\newblock \emph{Advances in Neural Information Processing Systems}, 34:\penalty0 17183--17193, 2021.

\bibitem[Maillard et~al.(2019)Maillard, Mann, Ortner, and Mannor]{maillard_active_2019}
Maillard, O.-A., Mann, T., Ortner, R., and Mannor, S.
\newblock Active {Roll}-outs in {MDP} with {Irreversible} {Dynamics}.
\newblock July 2019.

\bibitem[Mindermann et~al.(2018)Mindermann, Shah, Gleave, and Hadfield-Menell]{mindermann_active_2018}
Mindermann, S., Shah, R., Gleave, A., and Hadfield-Menell, D.
\newblock Active {Inverse} {Reward} {Design}.
\newblock In \emph{Proceedings of the 1st {Workshop} on {Goal} {Specifications} for {Reinforcement} {Learning}}, 2018.

\bibitem[Moldovan \& Abbeel(2012)Moldovan and Abbeel]{moldovan_safe_2012}
Moldovan, T.~M. and Abbeel, P.
\newblock Safe exploration in {Markov} decision processes.
\newblock In \emph{Proceedings of the 29th {International} {Conference} on {Machine} {Learning}}, {ICML}'12, pp.\  1451--1458, Madison, WI, USA, June 2012. Omnipress.

\bibitem[Mouton et~al.(2024)Mouton, Lucas, and Guest]{mouton2024operational}
Mouton, C., Lucas, C., and Guest, E.
\newblock The operational risks of {AI} in large-scale biological attacks.
\newblock Technical report, RAND Corporation, Santa Monica, 2024.

\bibitem[Osa et~al.(2018)Osa, Pajarinen, Neumann, Bagnell, Abbeel, and Peters]{osa_algorithmic_2018}
Osa, T., Pajarinen, J., Neumann, G., Bagnell, J., Abbeel, P., and Peters, J.
\newblock \emph{An {Algorithmic} {Perspective} on {Imitation} {Learning}}.
\newblock Foundations and trends in robotics. Now Publishers, 2018.

\bibitem[Qui{\~n}onero-Candela et~al.(2022)Qui{\~n}onero-Candela, Sugiyama, Schwaighofer, and Lawrence]{quinonero2022dataset}
Qui{\~n}onero-Candela, J., Sugiyama, M., Schwaighofer, A., and Lawrence, N.~D.
\newblock \emph{Dataset shift in machine learning}.
\newblock MIT Press, 2022.

\bibitem[Rajpurkar et~al.(2022)Rajpurkar, Chen, Banerjee, and Topol]{rajpurkar2022ai}
Rajpurkar, P., Chen, E., Banerjee, O., and Topol, E.~J.
\newblock {AI} in health and medicine.
\newblock \emph{Nature medicine}, 28\penalty0 (1):\penalty0 31--38, 2022.

\bibitem[Russo et~al.(2024)Russo, Celli, Colini~Baldeschi, Fusco, Haimovich, Karamshuk, Leonardi, and Tax]{russo2024online}
Russo, M., Celli, A., Colini~Baldeschi, R., Fusco, F., Haimovich, D., Karamshuk, D., Leonardi, S., and Tax, N.
\newblock Online learning with sublinear best-action queries.
\newblock \emph{Advances in Neural Information Processing Systems}, 37:\penalty0 40407--40433, 2024.

\bibitem[Shalev-Shwartz \& Ben-David(2014)Shalev-Shwartz and Ben-David]{shalev-shwartz_understanding_2014}
Shalev-Shwartz, S. and Ben-David, S.
\newblock \emph{Understanding {Machine} {Learning}: {From} {Theory} to {Algorithms}}.
\newblock Cambridge University Press, 1 edition, May 2014.

\bibitem[Slivkins(2011)]{slivkins_contextual_2011}
Slivkins, A.
\newblock Contextual {Bandits} with {Similarity} {Information}.
\newblock In \emph{Proceedings of the 24th {Annual} {Conference} on {Learning} {Theory} ({COLT})}, pp.\  679--702, December 2011.
\newblock ISSN: 1938-7228.

\bibitem[Slivkins et~al.(2019)]{slivkins2019introduction}
Slivkins, A. et~al.
\newblock Introduction to multi-armed bandits.
\newblock \emph{Foundations and Trends{\textregistered} in Machine Learning}, 12\penalty0 (1-2):\penalty0 1--286, 2019.

\bibitem[Spielman \& Teng(2004)Spielman and Teng]{spielman2004smoothed}
Spielman, D.~A. and Teng, S.-H.
\newblock Smoothed analysis of algorithms: Why the simplex algorithm usually takes polynomial time.
\newblock \emph{Journal of the ACM (JACM)}, 51\penalty0 (3):\penalty0 385--463, 2004.

\bibitem[Stradi et~al.(2024)Stradi, Castiglioni, Marchesi, and Gatti]{stradi2024learning}
Stradi, F.~E., Castiglioni, M., Marchesi, A., and Gatti, N.
\newblock Learning adversarial {MDP}s with stochastic hard constraints.
\newblock \emph{arXiv preprint arXiv:2403.03672}, 2024.

\bibitem[Turchetta et~al.(2016)Turchetta, Berkenkamp, and Krause]{turchetta_safe_2016}
Turchetta, M., Berkenkamp, F., and Krause, A.
\newblock Safe {Exploration} in {Finite} {Markov} {Decision} {Processes} with {Gaussian} {Processes}.
\newblock In \emph{Advances in {Neural} {Information} {Processing} {Systems}}, volume~29. Curran Associates, Inc., 2016.

\bibitem[Vapnik \& Chervonenkis(1971)Vapnik and Chervonenkis]{vapnik_uniform_1971}
Vapnik, V.~N. and Chervonenkis, A.~Y.
\newblock On the {Uniform} {Convergence} of {Relative} {Frequencies} of {Events} to {Their} {Probabilities}.
\newblock \emph{Theory of Probability \& Its Applications}, 16\penalty0 (2):\penalty0 264--280, January 1971.
\newblock Publisher: Society for Industrial and Applied Mathematics.

\bibitem[Villasenor \& Foggo(2020)Villasenor and Foggo]{villasenor2020artificial}
Villasenor, J. and Foggo, V.
\newblock Artificial intelligence, due process and criminal sentencing.
\newblock \emph{Michigan State Law Review}, pp.\  295--354, 2020.

\bibitem[Wachi et~al.(2024)Wachi, Shen, and Sui]{wachi_survey_2024}
Wachi, A., Shen, X., and Sui, Y.
\newblock A {Survey} of {Constraint} {Formulations} in {Safe} {Reinforcement} {Learning}.
\newblock volume~9, pp.\  8262--8271, August 2024.
\newblock ISSN: 1045-0823.

\bibitem[Wu(2020)]{wu_lecture_2020}
Wu, Y.
\newblock \emph{Lecture notes on: {Information}-theoretic methods for high-dimensional statistics}.
\newblock 2020.

\bibitem[Yang et~al.(2024)Yang, Zhou, Li, and Liu]{yang2024generalized}
Yang, J., Zhou, K., Li, Y., and Liu, Z.
\newblock Generalized out-of-distribution detection: A survey.
\newblock \emph{International Journal of Computer Vision}, pp.\  1--28, 2024.

\bibitem[Zhao et~al.(2023)Zhao, He, Chen, Wei, and Liu]{zhao_state-wise_2023}
Zhao, W., He, T., Chen, R., Wei, T., and Liu, C.
\newblock State-wise {Safe} {Reinforcement} {Learning}: {A} {Survey}.
\newblock volume~6, pp.\  6814--6822, August 2023.
\newblock ISSN: 1045-0823.

\end{thebibliography}
\bibliographystyle{icml2025}

\appendix

\onecolumn

\section{Proof of Theorem~\ref{thm:neg}}\label{sec:neg-proof}

\subsection{Proof notation}
\begin{enumerate}[leftmargin=1.5em,
    topsep=0.5ex,
    partopsep=0pt,
    parsep=0pt,
    itemsep=0.5ex]
    \item Let $M_j$ be the set of time steps $t \le T$ where $|m_j - x_t| \le \mfrac{1}{4f(T)}$. In words, $x_t$ is relatively close to the midpoint of $X_j$. This will imply that the suboptimal action is in fact quite suboptimal. This also implies that $x_t$ is in $X_j$, since each $X_j$ has length $1/f(T)$.
    \item Let $J_{\neg Q} = \{j \in [f(T)]: x_t \not\in X_j\ \forall t \in Q_T\}$ be the set of sections that are never queried. Since each query appears in exactly one section (because each input appears in exactly one section), $|J_{\neg Q}| \ge f(T) - |Q_T|$.
    \item For each $j \in J_{\neg Q}$, let $y^j$ be the most frequent action among time steps in $M_j$: $y^j = \argmax_{y \in \{0,1\}} | \{t \in M_j: y = y_t\} |$.
    \item Let $J_{\neg Q}' = \{j \in J_{\neg Q}: a_j \ne y^j\}$ be the set of sections where the more frequent action is wrong according to $\mufx$.
    \item Let $M_j' = \{t \in M_j: y_t \ne a_j\}$ be the set of time steps where the agent chooses the wrong action according to $\mufx$, and $x_t$ is close to the midpoint of section $j$.
\end{enumerate}

Since $\s, \A$, and $\x$ are random variables, all variables defined on top of them (such as $M_j$) are also random variables. In contrast, the partition $\X = \{X_1,\dots,X_{f(T)}\}$ and properties thereof (like the midpoints $m_j$) are not random variables.

\subsection{Proof roadmap}

The proof considers an arbitrary algorithm with sublinear queries, and proceeds via the following steps: 
\begin{enumerate}[leftmargin=1.5em,
    topsep=0.5ex,
    partopsep=0pt,
    parsep=0pt,
    itemsep=0.5ex]
\item  Show that multiplicative regret and additive regret are tightly related (\Cref{lem:prod-vs-add}). We will also use this lemma for our positive results.
\item Prove an asymptotic density lemma which we will use to show that $f(T) = \sqrt{|Q_T|T}$ is asymptotically between $|Q_T|$ and $T$ (Lemma~\ref{lem:density}).
\item  Prove a simple variant of the Chernoff bound which we will apply multiple times (Lemma~\ref{lem:chernoff}).
\item Show that with high probability, $\sum_{j \in S} |M_j|$ is large for any subset of sections $S$ (Lemma~\ref{lem:concentrate}).
\item Prove that $|J_{\neg Q}'|$ is large with high probability (\Cref{lem:neg-j-large}).
\item The key lemma is Lemma~\ref{lem:neg-error}, which shows that a randomly sampled $\x$ produces poor agent performance with high probability. The central idea is that at least $f(T) - |Q_T|$ sections are never queried (which is large, by Lemma~\ref{lem:density}), so the agent has no way of knowing the optimal action in those sections. As a result, the agent picks the wrong answer at least half the time on average (and at least a quarter of the time with high probability). Lemma~\ref{lem:concentrate} implies that a constant fraction of those time steps will have significantly suboptimal payoffs, again with high probability.
\item  Apply $\sup\limits_{\bfmu,\pi^m}\ \E\limits_{\s,\A}\ \rmul(\s, \A, \bfmu, \pi^m) \ge \E\limits_{\pi^m, \x \sim U(\{0,1\}^{f(T)})}\ \E\limits_{\s,\A}\ \rmul(\s, \A, \mufx, \pi^m)$. Here $U(\{0,1\}^{f(T)})$ is the uniform distribution over bit strings of length $f(T)$ and we write $\pi^m, \x \sim U(\{0,1\}^{f(T)})$ with slight abuse of notation, since $\pi^m$ is not drawn from $U(\{0,1\}^{f(T)})$ but rather is determined by $\x$ which is drawn from $U(\{0,1\}^{f(T)})$.
\item The analysis above results in a lower bound on $\rplus$. The last step is to use \Cref{lem:prod-vs-add} to obtain a lower bound on $\rmul$.
\end{enumerate}

Step 7 is essentially an application of the probabilistic method: if a randomly chosen $\mufx$ has high expected regret, then the worst-case $\bfmu$ also has high expected regret. We have included subscripts in the expectations above to distinguish between the randomness over $\x$ and $\s,\A$. When subscripts are omitted, the expected value is over all randomness, i.e., $\x,\s,$ and $\A$.

\subsection{Proof}

\begin{lemma}\label{lem:prod-vs-add}
If $\mu_t^m(x_t) \ge \mu_t(x_t,y_t)$ for all $t$, then $\rplus \le \rmul$. If $\mu_t(x_t,y_t) > 0$ for all $t$, then $\rmul \le \mfrac{\rplus}{\min_{t\in[T]} \mu_t(x_t,y_t)}$.
\end{lemma}
\begin{proof}
Recall the standard inequalities $1-\frac{1}{a} \le \log a \le a-1$ for any $a > 0$.

\textbf{Part 1: $\rplus \le \rmul$.} If $\mu_t(x_t,y_t) = 0$ for any $t \in [T]$, then $\rmul=\infty$ and the claim is trivially satisfied. Thus assume $\mu_t(x_t,y_t) > 0$ for all $t \in [T]$. 
\begin{align*}
\rplus =&\  \sum_{t=1}^T \mu_t^m(x_t) -\sum_{t=1}^T \mu_t(x_t,y_t)  &&  (\text{Definition of $\rplus$})\\
=&\  \sum_{t=1}^T \frac{\mu_t^m(x_t) -\mu_t(x_t,y_t)}{\mu_t^m(x_t)}  &&  (\text{$ \mu_t(x_t,y_t) \le \mu_t^m(x_t)$ and $0 \le \mu_t^m(x_t) \le 1$})\\
\le&\  \sum_{t=1}^T  \log\left(\frac{\mu_t^m(x_t)}{\mu_t(x_t,y_t)}\right)  && \Big(\text{$1-\frac{1}{a}\le \log a$ for any $a > 0$}\Big)\\
=&\  \log \prod_{t=1}^T  \mu_t^m(x_t) - \log \prod_{t=1}^T  \mu_t(x_t,y_t) && (\text{Properties of logarithms})\\
=&\ \rmul && (\text{Definition of $\rmul$})
\end{align*}
\textbf{Part 2: $\rmul \le \rplus/ \min_{t\in[T]} \mu_t(x_t,y_t)$}. We have
\begin{align*}
\rmul =&\ \log \prod_{t=1}^T  \mu_t^m(x_t) - \log \prod_{t=1}^T  \mu_t(x_t,y_t) && (\text{Definition of $\rmul$ given $\mu_t(x_t,y_t) > 0\ \forall t \in [T]$})\\
=&\ \sum_{t=1}^T  \log\left(\frac{\mu_t^m(x_t)}{\mu_t(x_t,y_t)}\right) && (\text{Properties of logarithms})\\
\le&\ \sum_{t=1}^T  \left(\frac{\mu_t^m(x_t) - \mu_t(x_t,y_t)}{\mu_t(x_t,y_t)}\right) && (\text{$\log a \le a-1$ for any $a > 0$})\\
\le&\  \sum_{t=1}^T  \frac{\mu_t(x_t,y_t)}{\min_{i\in[T]} \mu_i(x_i,y_i)}\left(\frac{\mu_t^m(x_t) - \mu_t(x_t,y_t)}{\mu_t(x_t,y_t)}\right) && (\text{$\mu_t(x_t,y_t) \ge \min_{i \in [T]} \mu_i(x_i,y_i) > 0\ \forall t \in [T]$})\\
=&\  \frac{1}{\min_{t\in[T]} \mu_t(x_t,y_t)}\sum_{t=1}^T  (\mu_t^m(x_t) - \mu_t(x_t,y_t)) && (\text{Arithmetic})\\
=&\  \frac{\rplus}{\min_{t\in[T]} \mu_t(x_t,y_t)} && (\text{Definition of $\rplus$})
\end{align*}
as claimed.
\end{proof}

\begin{lemma}\label{lem:density}
Let $a, b: \bbrspos \to \bbrspos$ be functions such that $a(x) \in o(b(x))$. Then $c(x) = \sqrt{a(x)b(x)}$ satisfies $a(x) \in o(c(x))$ and $c(x) \in o(b(x))$.
\end{lemma}

\begin{proof}
Since $a$ and $b$ are strictly positive (and thus $c$ is as well), we have
\[
\frac{a(x)}{c(x)} = \frac{a(x)}{\sqrt{a(x)b(x)}} = \sqrt{\frac{a(x)}{b(x)}} = \frac{\sqrt{a(x)b(x)}}{b(x)} = \frac{c(x)}{b(x)}
\]
Then $a(x) \in o(b(x))$ implies
\begin{align*}
\lim_{x\to\infty} \frac{a(x)}{c(x)} = \lim_{x\to\infty} \frac{c(x)}{b(x)} = \lim_{x\to\infty} \sqrt{\frac{a(x)}{b(x)}} = 0
\end{align*}
as required.
\end{proof}

\begin{lemma}
    \label{lem:chernoff}
Let $z_1,\dots,z_n$ be i.i.d. variables in $\{0,1\}$ and let $Z = \sum_{i=1}^n z_i$. If $\E[Z] \ge W$, then $\Pr\big[Z \le W/2\big] \le \exp(-W/8)$.
\end{lemma}
\begin{proof}
By the Chernoff bound for i.i.d. binary variables, we have $\Pr[Z \le \E[Z]/2] \le \exp(-\E[Z]/8)$. Since $-\E[Z] \le -W$ and $\exp$ is an increasing function, we have $\exp(-\E[Z]/8) \le \exp(-W/8)$. Also, $W/2 \le E[Z]/2$ implies $\Pr[Z \le W/2] \le \Pr[Z \le \E[Z]/2]$. Combining these inequalities proves the lemma.
\end{proof}

\begin{lemma}\label{lem:concentrate}
Let $S \subseteq [f(T)]$ be any nonempty subset of sections. Then
\[
\Pr\left[\sum_{j \in S} |M_j| \le \frac{T|S|}{4f(T)}\right] \le \exp\left(\frac{-T}{16f(T)}\right)
\]
\end{lemma}

\begin{proof}
Fix any $j \in [f(T)]$. For each $t \in [T]$ , define the random variable $z_t$ by $z_t = 1$ if $t \in M_j$ for some $j \in S$ and 0 otherwise. We have $t \in M_j$ iff $x_t$ falls within a particular interval of length $\frac{1}{2f(T)}$. Since these intervals are disjoint for different $j$'s, we have $z_t = 1$ iff $x_t$ falls within a portion of the input space with total measure $\frac{|S|}{2f(T)}$. Since $x_t$ is uniformly random across $[0,1]$, we have $\E[z_t] = \frac{|S|}{2f(T)}$. Then $\E[\sum_{t=1}^T z_t] = \E[\sum_{j \in S} |M_j|] = \frac{T|S|}{2f(T)}$. Furthermore, since $x_1,\dots,x_T$ are i.i.d., so are $z_1,\dots,z_T$. Then by Lemma~\ref{lem:chernoff},
\[
\Pr\left[\sum_{j \in S} |M_j| \le \frac{T|S|}{4f(T)}\right] \le \exp\left(\frac{-T|S|}{16f(T)}\right) \le \exp\left(\frac{-T}{16f(T)}\right)
\]
with the last step due to $|S| \ge 1$.
\end{proof}

\begin{lemma}
\label{lem:neg-j-large}
We have
\[
\Pr\left[|J_{\neg Q}'| \le \frac{f(T) - \E[|Q_T|]}{4}\right] \le \exp\left(-\frac{f(T) - \E[|Q_T|]}{16}\right)
\]
\end{lemma}
\begin{proof}
Define a random variable $z_j = \bfone_{j \in J_{\neg Q}'}$ for each $j \in J_{\neg Q}$. By definition, if $j \in J_{\neg Q}$, no input in $X_j$ is queried. Since queries outside of $X_j$ provide no information about $a_j$, the agent's actions must be independent of $a_j$. In particular, the random variables $a_j$ and $y^j$ are independent. Combining that independence with $\Pr[a_j = 0] = \Pr[a_j = 1] = 0.5$ yields $\Pr[z_j = 1] = 0.5$ for all $j \in J_{\neg Q}$. Then
\begin{align*}
\E\left[|J_{\neg Q}'|\right] =&\ \E\left[\sum_{j \in J_{\neg Q}} z_j\right]\\
=&\ |J_{\neg Q}|/2\\
\ge&\ \frac{f(T) - \E[|Q_T|]}{2}
\end{align*}
Furthermore, since $a_1,\dots,a_{f(T)}$ are independent, the random variables $\{z_j: j \in J_{\neg Q}\}$ are also independent. Applying Lemma~\ref{lem:chernoff} yields the desired bound.
\end{proof}
\begin{lemma}\label{lem:neg-error}
Suppose $f:\bbn\to \bbn$ and independently sample $\x \sim U(\{0,1\}^{f(T)})$ and $\s \sim U(\X)^T$.\footnote{That is, the entire set  $\{a_1,\dots,a_{f(T)},x_1,\dots,x_T\}$ is mutually independent.}  Then with probability at least $1 - \exp\big(\frac{-T}{16f(T)}\big) - \exp\big(-\frac{f(T) - \E[|Q_T|]}{16}\big)$,
\begin{align*}
\rplus \ge \frac{LT(f(T) - \E[|Q_T|])}{2^7 f(T)^2}
\end{align*}
\end{lemma}

\begin{proof}
Consider any $j \in J_{\neg Q}'$ and $t \in M_j' \subseteq M_j$. By definition of $M_j$, we have $|m_j - x_t| \le \frac{1}{4f(T)}$. Then by the definition of $\mufx$,
\begin{align*}
\mufx(x_t, y_t) =&\ 1 - L \left(\frac{1}{2f(T)} - |x_t - m_j|\right)\\
\le&\ 1 - L \left(\frac{1}{2f(T)} - \frac{1}{4f(T)}\right)\\
=&\ 1 - \frac{L}{4f(T)} 
\end{align*}
Since $\mufx^m(x_t) \ge \mufx(x_t,y_t)$ always, we can safely restrict ourselves to time steps $t \in M_j'$ for some $j \in J_{\neg Q}'$ and still obtain a lower bound:
\begin{align*}
\rplus =&\ \sum_{t=1}^T (\mufx^m(x_t) - \mufx(x_t,y_t)) && (\text{Definition of $\rplus$})\\
\ge&\ \sum_{j \in J_{\neg Q}'} \sum_{t\in M_j'} \big(\mufx^m(x_t) -  \mufx(x_t,y_t)\big) && (\mufx^m(x_t) \ge \mufx(x_t,y_t))\\
\ge&\ \sum_{j \in J_{\neg Q}'} \sum_{t\in M_j'} (1 - \mufx(x_t,y_t)) && (\mufx^m(x_t)=1 \text{ always})\\
\ge&\ \sum_{j \in J_{\neg Q}'} \sum_{t\in M_j'} \left(1 - 1 + \frac{L}{4f(T)} \right) && (\text{bound on $\mufx(x_t,y_t)$ for $t \in M_j'$})\\
=&\ \sum_{j \in J_{\neg Q}'} \frac{L|M_j'|}{4f(T)}  && (\text{Simplifying inner sum})
\end{align*}
 Since $j \in J_{\neg Q}$, the mentor is not queried on any time step $t \in M_j$, so $y_t \in \{0,1\}$ for all $t \in M_j$. Since the agent chooses one of two actions for each $t \in M_j$, the more frequent action must be chosen at least half of the time: $y_t = y^j$ for at least half of the time steps in $M_j$. Since $a_j \ne y^j$ for $j \in J_{\neg Q}'$, we have $y_t = y^j \ne a_j$ for those time steps, so $|M_j'| \ge |M_j|/2$. Thus
 \[
 \rplus \ge \sum_{j \in J_{\neg Q}'} \frac{L|M_j|}{8f(T)}
 \]
By Lemma~\ref{lem:concentrate}, \Cref{lem:neg-j-large}, and the union bound, with probability at least $1 - \exp\big(\frac{-T}{16f(T)}\big) - \exp\big(-\frac{f(T) - \E[|Q_T|]}{16}\big)$ we have $\sum_{j \in J_{\neg Q}'} |M_j| \ge \mfrac{T|J_{\neg Q}'|}{4f(T)}$ for all $j \in [f(T)]$ and $|J_{\neg Q}'| \ge \mfrac{f(T) - \E[|Q_T|]}{4}$. Assuming those inequalities hold, we have
\begin{align*}
\rplus \ge&\ \sum_{j \in J_{\neg Q}'} \frac{L|M_j|}{8f(T)}\\
\ge&\ \frac{L}{8f(T)} \cdot \frac{T|J_{\neg Q}'|}{4f(T)}\\
\ge&\ \frac{L}{8f(T)} \cdot \frac{T}{4f(T)} \cdot \frac{f(T) - \E[|Q_T|]}{4}\\
=&\ \frac{LT(f(T) - \E[|Q_T|])}{2^7 f(T)^2}
\end{align*}
as required.
\end{proof}
For a given $f:\bbn\to\bbn$, define $\alpha_f(T) = \exp\big(\frac{-T}{16f(T)}\big) + \exp\big(-\frac{f(T) - \E[|Q_T|]}{16}\big)$ for brevity.
\thmNeg*

\begin{proof}
If the algorithm has sublinear queries, then there exists $g(T) \in o(T)$ such that $\sup_{\bfmu,\pi^m} \E_{\s,\A}[|Q_T|] \le g(T)$. Consider any such $g(T)$ satisfying $g(T) > 0$. Since this holds for every $\bfmu$, it also holds in expectation over $\x \sim U(\{0,1\})^{f(T)}$, so $\E_{\x,\s,\A}[|Q_T|] = \E[|Q_T|] \le g(T)$. 

Next, Lemma~\ref{lem:density} gives us  $g(T) \in o(\sqrt{g(T)T})$ and $\sqrt{g(T)T} \in o(T)$. Let $f(T) = \lceil \sqrt{g(T) T} \rceil$: then $f(T) \in \Theta(\sqrt{g(T)T})$, so $g(T) \in o(f(T))$ and $f(T) \in o(T)$. First, this implies that $\lim_{T\to\infty} \alpha_f(T) = 0$. Second, $g(T) \in o(f(T))$ implies that exists $T_0$ such that $ g(T) \le f(T)/2$ for all $T \ge T_0$. We also have $\rplus \ge 0$ always since $\mufx^m(x_t) \ge \mufx(x_t,y_t)$ always. Then for all $T \ge T_0$ we have
\begin{align*}
&\ \E_{\pi^m,\x \sim U(\{0,1\}^{f(T)})}\ \, \E_{\s,\A}\ \left[\rplus(\s, \A, \bfmu, \pi^m)\right]\\ 
\ge&\ \alpha_f(T) \cdot 0 + \big(1-\alpha_f(T)\big) \left(\frac{LT(f(T) - \E[|Q_T|]}{2^7f(T)^2}\right) && (\text{\Cref{lem:neg-error} and $\rplus\ge 0$})\\
\ge&\ \big(1-\alpha_f(T)\big) \left(\frac{LT(f(T) - g(T))}{2^7f(T)^2}\right) && (\text{$\E[Q_T] \le g(T)$})\\
\ge&\  \big(1-\alpha_f(T)\big) \left(\frac{LT}{2^8 f(T)}\right)&& (\text{$g(T) \le f(T)/2$})
\end{align*}
Since $\lim_{T\to\infty} \alpha_f(T) = 0$, 
\begin{align*}
\sup_{\bfmu, \pi^m}\ \E_{\s,\A}\ \big[\rplus(\s, \A, \bfmu, \pi^m)\big] \ge&\ \E_{\pi^m,\x \sim U(\{0,1\}^{f(T)})}\ \, \E_{\s,\A}\ \big[\rplus(\s, \A, (\mufx,\dots,\mufx), \pi^m)\big]\\
\ge&\ \big(1-\alpha_f(T)\big) \left(\frac{LT}{2^8 f(T)}\right)\\
\in &\ \Omega\left(\frac{LT}{f(T)}\right)\\
=&\ \Omega\left(L \sqrt{\frac{T}{g(T)}}\right)
\end{align*}
This holds for any $g(T) \in o(T)$ such that $\sup_{\bfmu} \E[|Q_T|] \le g(T)$ and $g(T) > 0$. Thus we can simply set  $g(T) = \sup_{\bfmu,\pi^m} \E[|Q_T|] + 1$, since $\sup_{\bfmu,\pi^m} \E[|Q_T|] $ is indeed a function of only $T$.

Since $\mufx^m(x_t) \ge \mufx(x_t,y_t)$ for all $t \in [T]$, \Cref{lem:prod-vs-add} implies that
\begin{align*}
\sup_{\bfmu,\pi^m}\ \E_{\s,\A}\ \big[\rmul(\s, \A, \bfmu, \pi^m)\big] \ge&\ \E_{\pi^m,\x \sim U(\{0,1\}^{f(T)})}\ \, \E_{\s,\A}\ \big[\rmul(\s, \A, (\mufx,\dots,\mufx), \pi^m)\big]\\
\in&\ \Omega\left(L \sqrt{\frac{T}{\sup_{\bfmu,\pi^m}\E[|Q_T|] + 1}}\right)
\end{align*}
completing the proof.\end{proof}

\corNegProb*

\begin{proof}
We have $\prod_{t=1}^T \mufx^m(x_t) = 1$ from our construction. Then \Cref{lem:prod-vs-add} implies that $\rmul \ge \rplus$, so 
\begin{align*}
\exp (-\rplus) \ge&\ \exp(-\rmul)\\
=&\ \exp\left( \log \prod_{t=1}^T\mufx(x_t,y_t) - \log 1 \right)\\
=&\ \prod_{t=1}^T\mufx(x_t,y_t)
\end{align*}
Then by \Cref{lem:neg-error}, with probability $1-\alpha_f(T)$,
\begin{align*}
\prod_{t=1}^T \mufx(x_t,y_t) \le \exp\left( -\frac{LT(f(T) - \E[|Q_T|])}{2^7 f(T)^2}\right) \le \exp\left( -\frac{LT}{2^8 f(T)}\right)
\end{align*}
Since $\prod_{t=1}^T \mufx(x_t,y_t) \le 1$ always, 
\begin{align*}
\lim_{T\to\infty} \E_{\pi^m,\x \sim U(\{0,1\}^{f(T)})} \ \E_{\s,\A}\ \left[\prod_{t=1}^T \mufx(x_t,y_t)\right] \le&\
\lim_{T\to\infty} \E_{\pi^m,\x \sim U(\{0,1\}^{f(T)})} \ \E_{\s,\A}\ \left[\prod_{t=1}^T \mufx(x_t,y_t)\right] \\
\le&\ \lim_{T\to\infty} \left(1 \cdot \alpha_f(T) + (1-\alpha_f(T))\cdot \exp\left( -\frac{LT}{2^8 f(T)}\right)\right) \\
\le&\ \lim_{T\to\infty} (1-\alpha_f(T))\cdot \lim_{T\to\infty} \exp\left( -\frac{LT}{2^8 f(T)}\right)\\
=&\ 1 \cdot \exp(-\infty)\\
=&\ 0
\end{align*}
Since this upper bound holds for a randomly chosen $\mufx,\pi^m$, the same upper bound holds for a worst-case choice of $\bfmu,\pi^m$ among $\bfmu,\pi^m$ which satisfy $\mu_t^m(x) = 1$ for all $t \in [T],x \in \X$. Formally,
\[
\lim_{T\to\infty}\ \sup_{\bfmu,\pi^m:\, \mu_t(x) = 1\, \forall t,x} \E\left[\prod_{t=1}^T \mu_t(x_t,y_t)\right] \le 0
\]
Since $\prod_{t=1}^T \mu_t(x_t,y_t) \ge 0$ always, the inequality above holds with equality.
\end{proof}

\section{Proof of Theorem~\ref{thm:pos-nd}}\label{sec:nd-proof}

\subsection{Context on \Cref{lem:russo}}\label{sec:russo}

Before diving into the main proof, we provide some context on \Cref{lem:russo} from \Cref{sec:pos-nd}:

\lemRusso*

\Cref{lem:russo} is a restatement and simplification of Lemma 3.5 in \citet{russo2024online} with the following  differences:
\begin{enumerate}[leftmargin=1.5em,
    topsep=0.5ex,
    partopsep=0pt,
    parsep=0pt,
    itemsep=0.5ex]
    \item They parametrize their algorithm by the expected number of queries $\hat{k}$ instead of the query probability $p = \hat{k} / T$. 
    \item They include a second parameter $k$, which is the eventual target number of queries for their unconditional query bound. In our case, an expected query bound is sufficient, so we simply set $k = \hat{k}$. 
    \item They provide a second bound which is tighter for small $k$; that bound is less useful for us so we omit it. 
    \item Their ``actions'' correspond to policies in our setting, not actions in $\Y$. Their number of actions $n$ corresponds to $|\tilpi|$.
    \item We include an expectation over both loss terms, while they only include an expectation over the agent's loss. This is because  an adaptive adversary may choose the loss function for time $t$ in a randomized manner. Since we eventually set $\ell(t,\pi) = \bfone(\pi(x_t) \ne \pi^m(x_t))$, the randomization in $\ell$ corresponds to the randomization in $x_t$.
\end{enumerate}
Altogether, since \citet{russo2024online} set $\eta = \max\Big(\frac{1}{T}\sqrt{\frac{\hat{k} \log n}{2}}, \frac{k\hat{k}}{\sqrt{2}T^2}\Big)$, we end up with $\eta = \max\Big(\sqrt{\frac{p\log |\tilpi|}{2T}},\: \frac{p^2}{\sqrt{2}} \Big)$. \Cref{alg:hedge} provides precise pseudocode for the \textsc{HedgeWithQueries} algorithm to which \Cref{lem:russo} refers.

\begin{algorithm}[tb]
\centering
\begin{algorithmic}
\FUNCTION{\textsc{HedgeWithQueries}$(p \in (0,1],\: \text{finite policy class }\tilpi,\: \text{unknown } \ell: [T] \times \tilpi \to [0,1])$}
\STATE $w(\pi) \gets 1$ for all $\pi \in \tilpi$
\STATE $\eta \gets \max\big(\sqrt{\frac{p\log |\tilpi|}{2T}},\: \frac{p^2}{\sqrt{2}} \big)$
\FOR{$t$ \textbf{from} $1$ \textbf{to} $T$}
        \STATE with probability $p:$ $\texttt{hedgeQuery} \gets \texttt{true}$
        \STATE with probability $1-p:$ $\texttt{hedgeQuery} \gets \texttt{false}$
        \IF{\texttt{hedgeQuery}}
                \STATE Query and observe $\ell(t, \pi)$ for all $\pi \in \tilpi$
                \STATE $\ell^* \gets \min_{\pi \in \tilpi} \ell(t, \pi)$
                \STATE $w(\pi) \gets w(\pi)\cdot \exp(-\eta(\ell(t, \pi) - \ell^*))$ for all $\pi \in \tilpi$
                \STATE Select policy  $\arg\min_{\pi \in \tilpi} \ell(t, \pi)$        
        \ELSE
                \STATE $P(\pi) \gets w(\pi) / \sum_{\pi' \in \tilpi} w(\pi')$ for all $ \pi \in \tilpi$
                \STATE Sample $\pi_t \sim P$
                \STATE Select policy $\pi_t$
        \ENDIF
\ENDFOR
\ENDFUNCTION
\end{algorithmic}
\caption{A variant of the Hedge algorithm which only observes losses in response to queries.}
\label{alg:hedge}
\end{algorithm}

\subsection{Main proof}

We use the following notation throughout the proof:
\begin{enumerate}[leftmargin=1.5em,
    topsep=0.5ex,
    partopsep=0pt,
    parsep=0pt,
    itemsep=0.5ex]
    \item For each $t \in [T]$, let $S_t$ refer to the value of $S$ at the start of time step $t$.
    \item Let $M_T = \{t \in [T]: \pi_t(x_t) \ne \pi^m(x_t)\}$ be the set of time steps where Hedge's proposed action doesn't match the mentor's. Note that $|M_T|$ upper bounds the number of mistakes the algorithm makes (the number of mistakes could be smaller, since the algorithm sometimes queries instead of taking action $\pi_t(x_t)$). 
    \item  For $S \subseteq \X$, let $\vol(S)$ denote the $n$-dimensional Lebesgue measure of $S$.
    \item With slight abuse of notation, we will use inequalities of the form $f(T) \le g(T) + O(h(T))$ to mean that there exists a constant $C$ such that $f(T) \le g(T) + Ch(T)$.
    \item We will use ``Case 1'' to refer to finite VC dimension and $\sigma$-smooth $\s$ and ``Case 2'' to refer to finite Littlestone dimension.
\end{enumerate}

\begin{restatable}{lemma}{ndNumErrors}
\label{lem:nd-num-errors}
Let $\Y = \{0,1\}$. Assume $\pi^m \in \Pi$ where either (1) $\Pi$ has finite VC dimension $d$ and $\s$ is $\sigma$-smooth, or (2) $\Pi$ has finite Littlestone dimension $d$. Then for any $T \in \bbn$ and $\ep \ge 1/T$,\footnote{Note that this lemma omits the assumption of $\ep\le (\frac{\mum}{2L})^n$, since we do not need it for this lemma, and we would like to apply this lemma in the multi-action case without that assumption.} \Cref{alg:nd} satisfies
\[
\E [|M_T|] \in  O\left(\frac{d}{\sigma} T \ep \log (T + 1/\ep) \right)
\]
\end{restatable}

\begin{proof}
Define $\ell: [T] \times \tilpi \to [0,1]$ by $\ell(t, \pi) = \bfone(\pi(x_t) \ne \pi^m(x_t))$, and let $w^h$ and $\pi_t^h$ denote the values of $w$ and $\pi_t$ respectively in \textsc{HedgeWithQueries}, while $w$ and $\pi_t$ refer to the variables in \Cref{alg:nd}. Then $w$ and $w^h$ evolve in the exact same way, so the distributions of $\pi_t$ and $\pi_t^h$ coincide. Also, $\ep \ge 1/T$ implies that $ p =1/\sqrt{\ep T} \in (0,1]$. Thus by \Cref{lem:russo},
\begin{align*}
\sum_{t=1}^T \E[\ell(t, \pi_t)] - \min_{\stilpi \in \tilpi} \sum_{t=1}^T \E[\ell(t, \stilpi)] \le&\ \frac{\log |\tilpi| }{p^2}\\
=&\  T \ep \log |\tilpi| 
\end{align*}
Since $|M_T| = \sum_{t=1}^T \bfone(\pi_t(x_t) \ne \pi^m(x_t)) = \sum_{t=1}^T \ell(t, \pi_t)$, we have
\[
\E [|M_T|] \le  T \ep\log |\tilpi| + \min_{\stilpi \in \tilpi} \sum_{t=1}^T \E[\bfone(\stilpi(x_t) \ne \pi^m(x_t))]
\]
\textbf{Case 1:} Since $\tilpi$ is a smooth $\ep$-cover and $\pi^m \in \Pi$, \Cref{lem:smooth-concentrate} implies that $\E[\bfone(\stilpi(x_t) \ne \pi^m(x_t))] \le \ep/\sigma$ for any $\stilpi \in \tilpi$. Since $|\tilpi| \le (41/\ep)^d$ by construction (and such a $\tilpi$ is guaranteed to exist by \Cref{lem:smooth-cover}), we get
\begin{align*}
\E [|M_T|] \le&\  T \ep\log ((41/\ep)^d) + \min_{\stilpi \in \tilpi} \sum_{t=1}^T \frac{\ep}{\sigma}\\
=&\ d T \ep\log (41/\ep) + \frac{T\ep}{\sigma} \\
\in&\ O\left(\frac{d}{\sigma} T \ep \log (T + 1/\ep)\right)
\end{align*}
\textbf{Case 2:} Since $\tilpi$ is an adversarial cover of $\Pi$ and $\pi^m \in \Pi$, there exists $\stilpi \in \tilpi$ such that $\sum_{t=1}^T \bfone(\stilpi(x_t) \ne \pi^m(x_t)) = 0$. Since $|\tilpi| \le (eT/d)^d$ (with such a $\tilpi$ guaranteed to exist by \Cref{lem:adv-cover}), 
\begin{align*}
\E [|M_T|] \le&\ T \ep\log |\tilpi| + \min_{\stilpi \in \tilpi} \sum_{t=1}^T \bfone(\stilpi(x_t) \ne \pi^m(x_t))\\
\le&\ T \ep d \ln (eT/d)\\
\in&\ O\left(\frac{d}{\sigma} T \ep \log (T + 1/\ep)\right)
\end{align*}
as required.
\end{proof}

\begin{lemma}
\label{lem:nd-lipschitz-payoff}
For all $t \in [T]$, $\mu_t(x_t, y_t) \ge \mu_t^m(x_t) - L\ep^{1/n}$.
\end{lemma}
\begin{proof}
Consider an arbitrary $t \in [T]$. If $t \in Q_T$, then $\mu_t(x_t, y_t) = \mu_t^m(x_t)$ trivially, so assume $t\not\in Q_T$. Let $(x', y') =\argmin_{(x,y) \in S_t: \pi_t(x_t) = y} \norm{x_t-x}$. Since $t \not \in Q_T$, we must have $\norm{x_t - x'} \le \ep^{1/n}$.

We have $y' = \pi^m(x')$ by construction of $S_t$ and $\pi_t(x_t) = y'$ by construction of $y'$. Combining these with the local generalization assumption, we get
\begin{align*}
\mu_t(x_t, y_t) =&\ \mu_t(x_t, \pi_t(x_t))\\
=&\ \mu_t(x_t, \pi^m(x'))\\
\ge&\ \mu_t^m(x_t) - L\norm{x_t - x'}\\
\ge&\ \mu_t^m(x_t) - L\ep^{1/n}
\end{align*}
as required.
\end{proof}

\begin{lemma}
    \label{lem:pos-nd-regret}
Under the conditions of \Cref{thm:pos-nd-ep}, Algorithm~\ref{alg:nd} satisfies
\begin{align*}
\E\left[\rmul\right] \in&\ O\left(\frac{d L}{\sigma \mum}  T \ep^{1+1/n} \log (T + 1/\ep) \right)\\
\E\left[\rplus\right] \in&\ O\left(\frac{d L}{\sigma}  T \ep^{1+1/n} \log (T + 1/\ep) \right)
\end{align*}
\end{lemma}

\begin{proof}

We first claim that $y_t = \pi^m(x_t)$ for all $t \not \in M_T$. If $t \in Q_T$, the claim is immediate. If not, we have $y_t = \pi_t(x_t)$ by the definition of the algorithm and $\pi_t(x_t) = \pi^m(x_t)$ by the definition of $t \not\in M_T$. Thus $\mu_t(x_t, y_t) = \mu_t^m(x_t)$ for $t \not \in M_T$. For $t \in M_T$, \Cref{lem:nd-lipschitz-payoff} implies that $\mu_t^m(x_t) - \mu_t(x_t,y_t) \le L\ep^{1/n}$, so
\begin{align}
\rplus =&\ \sum_{t \in M_T} (\mu_t^m(x_t) - \mu_t(x_t, y_t)) \nonumber\\
\le&\ \sum_{t \in M_T}L\ep^{1/n} \nonumber\\
=&\ |M_T| L\ep^{1/n} \label{eq:rplus-upp}
\end{align}
Since $\ep \le \left(\mfrac{\mum}{2L}\right)^n$ by assumption, we have $L\ep^{1/n} \le \mum/2$ and thus $\mu_t(x_t,y_t) \ge \mu_t^m(x_t) - L\ep^{1/n} \ge \mum - \mum/2 = \mum / 2 > 0$ for all $t \in [T]$. Then by \Cref{lem:prod-vs-add},
\begin{align}
\rmul \le \frac{\rplus}{\mum/2} \le \frac{2|M_T| L\ep^{1/n}}{\mum} \label{eq:rmul-upp}
\end{align}
Taking the expectation and applying \Cref{lem:nd-num-errors} to Equations~\ref{eq:rplus-upp} and \ref{eq:rmul-upp} completes the proof.
\end{proof}

\begin{definition}
\label{def:packing}
Let $(K,\norm{\cdot})$ be a normed vector space and let $\delta > 0$. Then a multiset $S\subseteq K$ is a $\delta$-packing of $K$ if for all $a, b \in S$, $\norm{a-b} > \delta$. The $\delta$-packing number of $K$, denoted $\M(K,\norm{\cdot}, \delta)$, is the maximum cardinality of any $\delta$-packing of $K$.
\end{definition}

We only consider the Euclidean distance norm, so we just write $M(K, \norm{\cdot}, \delta) = M(K, \delta)$.

\begin{lemma}[Theorem 14.2 in \cite{wu_lecture_2020}]
    \label{lem:packing}
If $K \subset \bbr^n$ is convex, bounded, and contains a ball with radius $\delta > 0$, then
\[
\M(K,\delta) \le  \frac{3^n \vol(K)}{\delta^n \vol(B)}
\]
where $B$ is a unit ball.
\end{lemma}

\begin{lemma}[Jung's Theorem \citep{Jung1901}]
\label{lem:jung}
If $S \subset \bbr^n$ is compact, then there exists a closed ball with radius at most $\diam(S) \sqrt{\frac{n}{2(n+1)}}$ containing $S$.
\end{lemma}

\begin{lemma}\label{lem:pos-nd-queries}
Under the conditions of \Cref{thm:pos-nd-ep}, Algorithm~\ref{alg:nd} satisfies
\[
\E [|Q_T|] \in O\left(\sqrt{\frac{T}{\ep}} +\frac{d}{\sigma} T \ep \log (T + 1/\ep) +\frac{\E[\diam(\s)^n]}{\ep}\right)
\]
\end{lemma}

\begin{proof}
If $t \in Q_T$, then either $\texttt{hedgeQuery} = \texttt{true}$ or $\min_{(x,y) \in S_t: \pi_t(x_t) = y}\norm{x_t - x} > \ep^{1/n}$ (or both).
The expected number of time steps with $\texttt{hedgeQuery} = \texttt{true}$ is $p T = \sqrt{T/\ep}$. Let $\hat{Q} = \{t \in Q_T: \min_{(x,y) \in S_t: \pi_t(x_t) = y}\norm{x_t - x} > \ep^{1/n}\}$. We further subdivide $\hat{Q}$ into $\hat{Q}_1 = \{t\in \hat{Q}: \pi_t(x_t) \ne \pi^m(x_t)\}$ and $\hat{Q}_2 = \{t \in \hat{Q}: \pi_t(x_t) = \pi^m(x_t)\}$. Since $\hat{Q}_1 \subseteq M_T$, \Cref{lem:nd-num-errors} implies that $ \E[|\hat{Q}_1|] \in  O\left(\frac{d}{\sigma} T \ep \log(T + 1/\ep)\right)$.

Next, fix a $y \in \Y$ and let $X_y = \{x \in \s: \pi^m(x) = y\}$ be the multiset of observed inputs whose mentor action is $y$. Also let $\hats_2 = \{x_t: t \in \hat{Q}_2\}$ be the multiset of inputs associated with time steps in $\hat{Q}_2$. Note that $|\hats_2| = |\hat{Q}_2|$, since $\hats_2$ is a multiset. We claim that $\hats_2 \cap X_y$ is an $\ep^{1/n}$-packing of $X_y$. Suppose instead that there exists $x, x' \in \hats_2 \cap X_y$, with $\norm{x - x'} \le \ep^{1/n}$. WLOG assume $x$ was queried after $x'$ and let $t$ be the time step on which $x$ was queried. Since $x' \in \hats_2$, this implies $(x', \pi^m(x')) \in S_t$. Also, since $x, x' \in \hats_2$ we have $\pi_t(x_t) = \pi^m(x_t) = y = \pi^m(x')$. Therefore
\[
\min_{(x'', y'') \in S_t: y'' = \pi_t(x_t)} \norm{x_t - x''} \le \norm{x_t-x'} \le \ep^{1/n}
\]
which contradicts $t \in \hat{Q}$. Thus $\hats_2 \cap X_y$ is an $\ep^{1/n}$-packing of $X_y$.

By \Cref{lem:jung}, there exists a ball $B_1$ of radius $R:=\diam(\s)\sqrt{\frac{n}{2(n+1)}}$ which contains $\s$. Let $B_2$ be the ball with the same center as $B_1$ but with radius $\max(R, \ep^{1/n})$. Since $X_y \subset \s \subset B_1 \subset B_2$ and $\hats_2\cap X_y$ is an $\ep^{1/n}$-packing of $X_y$, $\hats_2 \cap X_y$ is also an $\ep^{1/n}$-packing of $B_2$. Also, $B_2$ must contain a ball of radius $\ep^{1/n}$, so Lemma~\ref{lem:packing} implies that 
\begin{align*}
|\hats_2 \cap X_y| \le&\ \M(B_2, \ep^{1/n})\\
\le&\ \frac{3^n \vol(B_2)}{\ep \vol(B)}\\
=&\ \big(\max(R, \ep^{1/n}) \big)^n\, \frac{3^n \vol(B)}{\ep \vol(B)}\\
=&\ \max\left(\diam(\s)^n \left(\frac{n}{2(n+1)}\right)^{n/2},\: \ep\right) \frac{3^n}{\ep}\\
\le&\ O\left(\frac{\diam(\s)^n}{\ep} +1\right)
\end{align*}
(The $+1$ is necessary for now since $\diam(\s)$ could theoretically be zero.) Since $\hats_2 \subseteq \{x_1,\dots,x_T\} \subseteq \cup_{y\in\Y}X_y$, we have $|\hats_2| \le \sum_{y \in \Y} |\hats_2\cap X_y|$ by the union bound. Therefore
\begin{align*}
\E [|Q_T|] \le&\ \sqrt{\frac{T}{\ep}} + \E[|\hat{Q}|] && \\
=&\ \sqrt{\frac{T}{\ep}} + \E[|\hat{Q}_1|] + \E[|\hat{Q}_2|]\\
=&\ \sqrt{\frac{T}{\ep}} + \E[|\hat{Q}_1|] + \E[|\hats_2|]\\
\le&\ \sqrt{\frac{T}{\ep}} + \E[|\hat{Q}_1|] + \E\left[\sum_{y \in \Y} |\hats_2\cap X_y|\right]\\\
\le&\ \sqrt{\frac{T}{\ep}} + O\left(\frac{d}{\sigma} T \ep \log (T + 1/\ep)\right) + \sum_{y \in \Y} O\left(\frac{\E[\diam(\s)^n]}{\ep} +1\right)  \\
\le&\ \sqrt{\frac{T}{\ep}} + O\left(\frac{d}{\sigma} T \ep \log (T + 1/\ep)\right) + |\Y| \cdot  O\left(\frac{\E[\diam(\s)^n]}{\ep} +1\right)  \\
\le&\ O\left(\sqrt{\frac{T}{\ep}} +\frac{d}{\sigma} T \ep \log (T + 1/\ep) +\frac{\E[\diam(\s)^n]}{\ep}\right)
\end{align*}
as required.
\end{proof}
\Cref{thm:pos-nd-ep} follows from Lemmas~\ref{lem:pos-nd-regret} and \ref{lem:pos-nd-queries}:

\thmNDep*

We then perform some arithmetic to obtain \Cref{thm:pos-nd}:

\thmND*

\begin{proof}
We have
\begin{align*}
\E \left[\rmul\right] \in&\  O\left(\frac{d L}{\sigma \mum} T^{1 - \frac{2n}{2n+1} -\frac{2}{2n+1}}  \left(\log T + \log (T^\frac{2n}{2n+1})\right)\right)\\
=&\ O\left(\frac{d L}{\sigma \mum} T^{\frac{-1}{2n+1}} \log T\right)
\end{align*}
and similarly for $\E[\rplus]$. For $\E[|Q_T|]$,
\begin{align*}
\E[|Q_T|] \in&\ O\left(\sqrt{T^{1+\frac{2n}{2n+1}}} +\frac{d}{\sigma} T^{1-\frac{-2n}{2n+1}} \left(\log T + \log (T^\frac{2n}{2n+1})\right)  + T^\frac{2n}{2n+1}\E[\diam(\s)^n]\right)\\
=&\ O\left(T^\frac{2n+0.5}{2n+1} +\frac{d}{\sigma} T^\frac{1}{2n+1} \log T + T^\frac{2n}{2n+1}\E[\diam(\s)^n]\right)\\
\le&\ O\left(T^\frac{4n+1}{4n+2}\Big(\frac{d}{\sigma} \log T + \E[\diam(\s)^n]\Big)\right)
\end{align*}
\end{proof}

\subsection{Adaptive adversaries}\label{sec:adaptive}

If $x_t$ is allowed to depend on the events of prior time steps, we say that the adversary is adaptive. In contrast, a non-adaptive or ``oblivious'' adversary must choose the entire input upfront. This distinction is not relevant for deterministic algorithms, since an adversary knows exactly how the algorithm will behave for any input. In other words, the adversary gains no new information during the execution of the algorithm. For randomized algorithms, an adaptive adversary can base the choice of $x_t$ on the results of randomization on previous time steps (but not on the current time step), while an oblivious adversary cannot.


In the standard online learning model, Hedge guarantees sublinear regret against both oblivious and adaptive adversaries (Chapter 5 of \citet{slivkins2019introduction} or Chapter 21 of \citet{shalev-shwartz_understanding_2014}). However, \citet{russo2024online} state their result only for oblivious adversaries. In order for our overall proof of \Cref{thm:pos-nd-ep} to hold for adaptive adversaries, \Cref{lem:russo} (Lemma 3.5 in \citealp{russo2024online}) must also hold for adaptive adversaries. In this section, we argue why the proof of \Cref{lem:russo} (Lemma 3.5 in their paper) goes through for adaptive adversaries as well. For the rest of \Cref{sec:adaptive}, lemma numbers refer to the numbering in \citet{russo2024online}.

\paragraph{The importance of independent queries.} Recall from \Cref{sec:russo} that \citet{russo2024online} allow two separate parameters $k$ and $\hat{k}$, which we unify for simplicity. Recall also that Lemma 3.5 refers to the variant of Hedge which queries with probability $p =\hat{k}/T = k/T$ independently on each time step (\Cref{alg:hedge}). More precisely, on each time step $t$, the algorithm samples $X_t \sim \text{Bernoulli}(p)$ and queries if $X_t = 1$. The key idea is that $X_t$ is independent of events on previous time steps. Thus even conditioning on the history up to time $t$, for any random variable $Y_t$ we can write
\[
\E[Y_t] = (1-p) \E[Y_t \mid X_t = 0] + p \E[Y_t \mid X_t = 1]
\]
This insight immediately extends Observation 3.3 to adaptive adversaries (with the minor modification that queries are now issued independently with probability $p$ on each time step instead of issuing $k$ uniformly distributed queries). Specifically, using the notation from \citet{russo2024online} where $i_t$ is the action chosen at time $t$, $i_t^0$ is the action chosen at time $t$ if a query is not issued, and $i_t^*$  is the optimal action at time $t$, we have
\begin{align*}
\E[\ell_t(i_t)] =&\ (1-p) \E[\ell_t(i_t^0)] + p \E[\ell_t(i_t^*)]\\
=&\ \left(1-\frac{k}{T}\right) \E[\ell_t(i_t^0)] + \frac{k}{T} \E[\ell_t(i_t^*)]
\end{align*}
The same logic applies to other statements like $\E[\hat{\ell}_t(i) \mid X_{\le t-1}, I_{\le t-1}] = \ell_t(i) - \ell_t(i_t^*)$ and immediately extends those statements to adaptive adversaries as well.

\paragraph{Applying Observation 3.3.} The other tricky part of the proof is applying Observation 3.3 using a new loss function $\hat{\ell}$ defined by $\hat{\ell}_t(i) = \frac{T}{\hat{k}}(\ell_t(i) - \ell_t(i_t^*)) \bfone(X_t = 1)$. To do so, we must argue that standard Hedge run on $\hat{\ell}$ is the ``counterpart without queries'' of \textsc{HedgeWithQueries}. Specifically, both algorithms must have the same weight vectors on every time step, and the only difference should be that \textsc{HedgeWithQueries} takes the optimal action on each time step independently with probability $p$ (and otherwise behaves the same as standard Hedge). On time steps with $X_t = 0$, standard Hedge observes $\hat{\ell}_t(i) = 0$ for all actions $i$ and thus makes no updates, and \textsc{HedgeWithQueries} makes no updates by definition. On time steps with $X_t = 1$, both algorithms perform the typical updates $w_{t+1}(i) = w_t(i) \cdot \exp(-\eta(\hat{\ell}_t(i) - \hat{\ell}_t(i_t^*)))$. Thus the weight vectors are the same for both algorithms on every time step. Furthermore, \textsc{HedgeWithQueries} takes the optimal action at time $t$ iff $X_t = 1$, which occurs independently with probability $p$ on each time step. Thus standard Hedge run on $\hat{\ell}$ is the ``counterpart without queries'' of \textsc{HedgeWithQueries}. Note that since $\hat{\ell}$ is itself a random variable, the law of iterated expectation is necessary to formalize this.

\paragraph{The rest of the proof.} The other elements of the proof of Lemma 3.5 are as follows:
\begin{enumerate}[leftmargin=1.5em,
    topsep=0.5ex,
    partopsep=0pt,
    parsep=0pt,
    itemsep=0.5ex]
\item Lemma 3.1, which analyzes the standard version of Hedge (i.e., no queries and losses are observed on every time step).
\item Applying Lemma 3.1 to $\hat{\ell}$.
\item Arithmetic and rearranging terms.
\end{enumerate}

The proof of Lemma 3.1 relies on simple arithmetic properties of the Hedge weights. Regardless of the adversary's behavior, $\hat{\ell}$ is a well-defined loss function, so Lemma 3.1 can be applied. Step 3 clearly has no dependence on the type of adversary. Thus we conclude that Lemma 3.5 extends to adaptive adversaries.

\section{Generalizing \Cref{thm:pos-nd} to many actions}\label{sec:multi}

We use the standard ``one versus rest'' reduction  (see, e.g., Chapter 29 of \citealp{shalev-shwartz_understanding_2014}). For each action $y$, we will learn a binary classifier which predicts whether action $y$ is the mentor's action. Formally, for each $y \in \Y$, define the policy class $\Pi_y = \{\pi_y: \pi \in \Pi \text{ and }\pi_y(x) = \bfone(\pi(x) = y)\ \forall x \in \X\}$. In words, for each policy $\pi: \X\to \Y$ in $\Pi$, there exists a policy $\pi_y: \X\to \{0,1\}$ in $\Pi_y$ such that $\pi_y(x) = \bfone(\pi(x) = y)$ for all $x \in \X$. 

\Cref{alg:multi} runs one copy of our binary-action algorithm (\Cref{alg:nd}) for each action $y \in \Y$. At each time step $t$, the copy for action $y$ returns an action $b_t^y$, with $b_t^y = 1$ indicating a belief that $y = \pi^m(x_t)$ and $b_t^y = 0$ indicating a belief that $y \ne \pi^m(x_t)$. (Note that $b_t^y = \hata$ is also possible, indicating that the mentor was queried.) 

The key idea is that if $b_t^y$ is correct for each action $y$, there will be exactly one $y$ such that $b_t^y = 1$, and specifically it will be $y = \pi^m(x_t)$. Thus we are guaranteed to take the mentor's action on such time steps. The analysis for \Cref{thm:pos-nd} (specifically, \Cref{lem:nd-num-errors}) bounds the number of time steps when a given copy of \Cref{alg:nd} is incorrect, so by the union bound, the number of time steps where \emph{any} copy is incorrect is $|\Y|$ times that bound. That in turn bounds the number of time steps where \Cref{alg:multi} takes an action other than the mentor's. Similarly, the number of queries made by \Cref{alg:multi} is at most $|\Y|$ times the bound from \Cref{thm:pos-nd}. The result is the following theorem:

\begin{algorithm}[tb]
\centering
\begin{algorithmic}
\STATE Inputs: $T \in \bbn,\: \ep \in \bbrspos,\: d \in \bbn,\,$ policy class $\Pi$
\FOR{$y \in \Y$}
\IF{$\Pi_y$ has VC dimension $d$}
    \STATE $\tilpi_y \gets$ any smooth $\ep$-cover of $\Pi_y$ of size at most $(41/\ep)^d$
\ELSIF{$\Pi_y$ has Littlestone dimension $d$}
    \STATE $\tilpi_y \gets$ any adversarial cover of $\Pi_y$ of size at most $(eT/d)^d$
\ENDIF
\ENDFOR
\FOR{$t$ \textbf{from} $1$ \textbf{to} $T$}
        \FOR{$y \in \Y$}
            \STATE $b_t^y \gets $ action at time $t$ from the copy of \Cref{alg:nd} running on $\Pi_y$ (with the same $T, \ep, d$)
         \ENDFOR
         \IF{$b_t^y \ne \hata$  $\forall y \in \Y$ and $\exists a\in \Y: b_t^y = 1$}
            \STATE Take any action $y$ with $b_t^y = 1$
         \ELSE
            \STATE Take an arbitrary action in $\Y$
        \ENDIF
\ENDFOR
\end{algorithmic}
\caption{extends \Cref{alg:nd} to many actions.}
\label{alg:multi}
\end{algorithm}

\begin{restatable}{theorem}{thmMulti}
\label{thm:multi}
Assume $\pi^m \in \Pi$ where either (1) $\Pi_y$ has finite VC dimension $d$ and $\s$ is $\sigma$-smooth or (2) $\Pi_y$ has finite Littlestone dimension $d$ for all $y \in \Y$. Then for any $T \in \bbn$, \Cref{alg:multi} with $\ep = T^\frac{-2n}{2n+1}$ satisfies
\begin{align*}
\E \left[\rmul\right] \in&\ O\left(\frac{|\Y| d L}{\sigma \mum} T^{\frac{-1}{2n+1}} \log T\right)\\ 
\E \left[\rplus\right] \in&\ O\left(\frac{|\Y| d L}{\sigma} T^{\frac{-1}{2n+1}} \log T\right)\\
\E[|Q_T|] \in&\ O\left(|\Y|T^\frac{4n+1}{4n+2}\left(\frac{d}{\sigma} \log T + \E[\diam(\s)^n]\right)\right) 
\end{align*}
\end{restatable}

We use the following terminology and notation in the proof of \Cref{thm:multi}:
\begin{enumerate}[leftmargin=1.5em,
    topsep=0.5ex,
    partopsep=0pt,
    parsep=0pt,
    itemsep=0.5ex]
    \item We refer to the copy of \Cref{alg:nd} running on $\Pi_y$ as ``copy $y$ of \Cref{alg:nd}''.
    \item Recall that $S_t$ refers to the value of $S$ at the start of time step $t$ in \Cref{alg:nd}. Let $\pi_t^y$ and $S_t^y$ refer to the values of $\pi_t$ and $S_t$ for copy $y$ of \Cref{alg:nd}.
    \item  Let $\pi^{my}: \X\to \{0,1\}$ be the policy defined by $\pi^{my}(x) = \bfone(\pi^m(x) = y)$. Note that querying the mentor tells the agent $\pi^m(x_t)$, which allows the agent to compute $\pi^{my}(x_t)$: this is necessary when \Cref{alg:nd} queries while running on some $\Pi_y$.
    \item  Let $M_T^y = \{t \in [T]: b_t^y \ne \pi^{my}(x_t)\}$ be the set of time steps where $\pi_t^y$ does not correctly determine whether the mentor would take action $y$ and let $M_T = \{t \in [T]: y_t \ne \pi^m(x_t)\}$  be the set of time steps where the agent's action does not match the mentor's.
\end{enumerate}

\begin{lemma}
 \label{lem:multi-errors}   
We have $|M_T| \le \sum_{y \in \Y} |M_T^y|$.
\end{lemma}

\begin{proof}
We claim that $M_T \subseteq \cup_{y \in \Y} M_T^y$. Suppose the opposite: then there exists $t \in M_T$ such that $b_t^y = \pi^{my}(x_t)$ for all $y \in \Y$. Since $\pi^m(x_t) \in \Y$, there is exactly one $y \in \Y$ such that $\bfone(\pi^m(x_t) = y) = \pi^{my}(x_t) =b_t^y = 1$. Specifically, this holds for $y = \pi^m(x_t)$. But then \Cref{alg:multi} takes action $\min\{y \in \Y: b_t^y = 1\} = \pi^m(x_t)$, which contradicts $t \in M_T$. Therefore $M_T \subseteq \cup_{y \in \Y} M_T^y$, and applying the union bound completes the proof.
\end{proof}

\begin{lemma}
\label{lem:multi-lipschitz}
For all $t \in [T]$, $\mu_t^m(x_t) - \mu_t(x_t, y_t) \le L \ep^{1/n}$.
\end{lemma}

\begin{proof}
The argument is similar to the proof of \Cref{lem:nd-lipschitz-payoff}. If $\mu_t^m(x_t) \ne \mu_t(x_t, y_t)$, then $y_t = y$ for some $y \in \Y$ where $b_t^y = 1$. Therefore copy $y$ of \Cref{alg:nd} did not query at time $t$ and $\pi_t^y(x_t) = 1$. Let $(x', y') =\argmin_{(x,y) \in S_t^y: \pi_t^y(x_t) = y} \norm{x_t-x}$. Then $\norm{x_t - x'} \le \ep^{1/n}$ and $y' = \pi_t^y(x_t) = 1$.

By construction of $S_t^y$, $y' = \pi^{my}(x')$ so $\pi^{my}(x') = 1$ which implies $\pi^m(x') = y$. Then by the local generalization assumption,
\[
\mu_t(x_t, y_t) = \mu_t(x_t, y) = \mu_t(x_t, \pi^m(x')) \ge \mu_t^m(x_t) - L\norm{x_t - x'} \ge \mu_t^m(x_t) - L\ep^{1/n}
\]
as required.
\end{proof}

We now proceed to the proof of \Cref{thm:multi}. 

\begin{proof}[Proof of \Cref{thm:multi}]
\Cref{thm:pos-nd} implies that each copy of \Cref{alg:nd} makes $O\big(T^\frac{4n+1}{4n+2}\big(\frac{d}{\sigma} \log T + \E[\diam(\s)^n]\big)\big)$ queries in expectation. Thus by linearity of expectation the expected number of queries made by \Cref{alg:multi} is $O\big(|\Y| T^\frac{4n+1}{4n+2}\big(\frac{d}{\sigma} \log T + \E[\diam(\s)^n]\big)\big)$.\footnote{This is an overestimate because the agent makes at most one query per time step, even if multiple copies request a query.} Similar to  the proof of \Cref{lem:pos-nd-regret}, we have 
\begin{align*}
\rplus =&\ \sum_{t \in M_T} (\mu_t^m(x_t) - \mu_t(x_t, y_t)) && (\text{$\mu_t^m(x_t) = \mu_t(x_t,y_t)$ for all $t\not\in M_T$}) \\
\le&\ \sum_{t \in M_T}L\ep^{1/n} && (\text{\Cref{lem:multi-lipschitz}})\\
=&\ |M_T| L\ep^{1/n} && (\text{Simplifying sum})\\
\le&\  L\ep^{1/n}\sum_{y\in \Y} |M_T^y| && (\text{\Cref{lem:multi-errors}})
\end{align*}
Since each copy satisfies the conditions of \Cref{lem:nd-num-errors}, we get
\begin{align*}
\E[\rplus] \le L\ep^{1/n} \sum_{y \in \Y} O\left(\frac{d}{\sigma} T \ep \log(1/\ep) \log T\right) = O\left(|\Y|L\ep^{1/n} \frac{d}{\sigma} T \ep \log(1/\ep) \log T\right)
\end{align*}
Since $\lim_{T\to\infty} \ep = 0$, there exists $T_0$ such that $L\ep^{1/n} \le \mum/2$ for all $T \ge T_0$. Then by \Cref{lem:prod-vs-add},
\begin{align}
\E[\rmul] \le \frac{\E[\rplus]}{\mum/2} \in O\left(|\Y|L\ep^{1/n} \frac{d}{\sigma \mum} T \ep \log(1/\ep) \log T\right)
\end{align}
Plugging $\ep = T^{\frac{-2n}{2n+1}}$ to the bounds above on $\E[\rplus]$ and $\E[\rmul]$ yields the desired bounds (see the arithmetic in the proof of \Cref{thm:pos-nd} in Appendix~\ref{sec:nd-proof}).
\end{proof}

\section{There exist policy classes which are learnable in our setting but not in the standard online model} \label{sec:positive-1d}

This section presents another algorithm with subconstant regret and sublinear queries, but under different assumptions. The primary takeaway here is that our algorithm can handle the class of thresholds on $[0,1]$, which is known to have infinite Littlestone dimension and thus be hard in the standard online learning model. (Example 21.4 in \citealp{shalev-shwartz_understanding_2014}).

Specifically, we assume a 1D input space and we allow the input sequence to be fully adversarial chosen. Instead of VC/Littlestone dimension, we consider the following notion of simplicity:

\begin{definition}\label{def:segments}
Given a mentor policy $\pi^m$, partition the input space $\X$ into intervals such that all inputs within each interval share the same mentor action. Let $\{X_1,\dots,X_k\}$ be a partition that minimizes the number of intervals. We call each $X_j$ a \emph{segment}. Let $S(\pi^m)$ denote the number of segments in $\pi^m$.
\end{definition}

Bounding the number of segments is similar conceptually to VC dimension in that it limits the ability of the policy class to realize arbitrary combinations of labels (i.e., mentor actions) on $\s$. For example, if $\Pi$ is the class of thresholds on $[0,1]$, every $\pi \in \Pi$ has at most two segments, and thus the positive result in this section will apply. This demonstrates the existence of policy classes which are learnable in our setting but not learnable in the standard online learning model, meaning that the two settings do not exactly coincide. 

Unlike our primary algorithm (\Cref{alg:nd}), this algorithm does require direct access to the input encoding. However, the point of this section is not to present a practical algorithm: it is simply to demonstrate that our setting and the standard online setting do not exactly coincide.

We prove the following regret bound. Like our previous results, this bound applies to both multiplicative and additive regret.

\begin{restatable}{theorem}{thmSimple}\label{thm:simple}
For any $\s \in \X^T$, any $\pi^m$ with $S(\pi^m) \le K$, and any function $g:\bbn\to\bbn$ satisfying $g(T) \ge 2L/\mum$, Algorithm~\ref{alg:1d} satisfies
\begin{align*}
\rmul \le&\ \frac{4LKT}{g(T)^2 \mum}\\
\rplus \le&\ \frac{2LKT}{ g(T)^2}\\
|Q_T| \le&\ (\diam(\s)+4)g(T)
\end{align*}
\end{restatable}

Choosing $g(T) = T^c$ for $c \in (1/2, 1)$ is sufficient for subconstant regret and sublinear queries:

\begin{theorem}\label{thm:simple-final}
For any $c \in (1/2,1)$, Algorithm~\ref{alg:1d} with $g(T) = T^c$ satisfies
\begin{align*}
\rmul \in&\ O\left(\frac{LKT^{1-2c}}{\mum}\right)\\
\rplus \in&\ O\left(LKT^{1-2c}\right)\\
|Q_T| \in&\ O(T^c (\diam(\s) +1))
\end{align*}
\end{theorem}


\renewcommand{\algorithmicendfor}{\textbf{end for}}%
\renewcommand{\algorithmicendif}{\textbf{end if}}%
\renewcommand{\algorithmicendfunction}{\textbf{end function}}%
\begin{algorithm}[ht!]
\centering
\begin{algorithmic}[1]
\FUNCTION{\textsc{DBWRQ}$(T \in \bbn,\: g: \bbn \to \bbn)$}
\STATE $X_Q \gets \emptyset$ \ (previously queried inputs)
\STATE $\pi \gets \emptyset$ \ (records $\pi^m(x)$ for each $x \in X_Q$)
\STATE $\B \gets \emptyset$ \ (The set of active buckets)
  \FOR{$t$ \textbf{from} $1$ \textbf{to} $T$}
  	\STATE \textsc{EvaluateInput}$(x_t)$
\ENDFOR
\ENDFUNCTION
\FUNCTION{\textsc{EvaluateInput}$(x \in \X)$}
    \IF {$x \not\in B$ for all $B \in \B$} 
        \STATE $B \gets \left[\frac{j-1}{g(T)}, \frac{j}{g(T)}\right]$ for $j \in \mathbb{Z}$ such that $x \in B$
        \STATE $\B \gets \B \cup \{B\}$
        \STATE $n_B \gets 0$ 
        \STATE \textsc{EvaluateInput}$(x)$
    \ELSE 
        \STATE $B \gets$ any bucket containing $x$
      	\IF {$X_Q \cap B = \emptyset$} 
    		\STATE Query mentor and observe $\pi^m(x)$
    		\STATE $\pi(x) \gets \pi^m(x)$ 
    		\STATE $X_Q \gets X_Q \cup \{x\}$
            \STATE $n_B \gets n_B + 1$
    	\ELSIF{$n_B < T/g(T)$} 
    		\STATE Let $x' \in X_Q \cap B$
    		\STATE Take action $\pi(x')$
    		\STATE $n_B \gets n_B + 1$
    	\ELSE 
    		\STATE $B = [a,b]$
    		\STATE $(B_1, B_2) \gets \Big(\left[a, \frac{a+b}{2}\right], \left[\frac{a+b}{2}, b\right]\Big)$
    	    \STATE $(x_{B_1}, x_{B_2}) \gets (0,0)$
    		\STATE $\B \gets \B \cup \{B_1, B_2\} \setminus B$
    		\STATE \textsc{EvaluateInput}$(x)$
    	\ENDIF
    \ENDIF
\ENDFUNCTION
\end{algorithmic}
\caption{achieves subconstant regret when the mentor's policy has a bounded number of segments.}
\label{alg:1d}
\end{algorithm}

\subsection{Intuition behind the algorithm}

We call our algorithm ``Dynamic Bucketing With Routine Querying", or DBWRQ (pronounced ``DBWRQ"). The algorithm maintains a set of buckets which partition the observed portion of the input space. Each bucket's length determines the maximum loss in payoff we will allow from that subset of the input space. As long as the bucket contains a query from a prior time step, local generalization allows us to bound $\mu_t^m(x_t) - \mu_t(x_t, y_t)$ based on the length of the bucket containing $x_t$. We always query if the bucket does not contain a prior query; in this sense the querying is ``routine''. 

The granularity of the buckets is controlled by a function $g$, with the initial buckets having length $1/g(T)$. Since we can expect one query per bucket, we need $g(T) \in o(T)$ to ensure sublinear queries.

Regardless of the bucket length, the adversary can still place multiple segments in the same bucket $B$. A single query only tells us the optimal action for one of those segments, so we risk a payoff as bad as $\mu_t^m(x_t) - O(\len(B))$ whenever we choose not to query. We can endure a limited number of such payoffs, but if we never query again in that bucket, we may suffer $\Theta(T)$ such payoffs. Letting $\mu_t^m(x_t) = 1$ for simplicity, that would lead to $\prod_{t=1}^T \mu_t(x_t, y_t) \le \big(1 - \frac{1}{O(g(T))}\big)^{\Theta(T)}$, which converges to 0 (i.e., guaranteed catastrophe) when $g(T) \in o(T)$.

This failure mode suggests a natural countermeasure: if we start to suffer significant (potential) losses in the same bucket, then we should probably query there again. One way to structure these supplementary queries is by splitting the bucket in half when enough time steps have involved that bucket. It turns out that splitting after $T/g(T)$ time steps is a sweet spot.

\subsection{Proof notation}\label{sec:1d-notation}

We will use the following notation throughout the proof of Theorem~\ref{thm:simple}:
\begin{enumerate}[leftmargin=1.5em,
    topsep=0.5ex,
    partopsep=0pt,
    parsep=0pt,
    itemsep=0.5ex]
\item Let $M_T = \{t \in [T]: \mu_t(x_t, y_t) < \mu_t^m(x_t)\}$ be the set of time steps with a suboptimal payoff.
\item Let $B_t$ be the bucket that is used on time step $t$ (as defined on line 16 of \Cref{alg:1d}).
\item Let $d(B)$ be the \emph{depth} of bucket $B$.
\begin{enumerate}
    \item Buckets created on line 11 are depth 0.
    \item We refer to $B_1, B_2$ created on line 28 as the children of the bucket $B$ defined on line 16.
    \item If $B'$ is the child of $B$, $d(B') = d(B) + 1$.
    \item Note that $\len(B) = \mfrac{1}{g(T) 2^{d(B)}}$.
\end{enumerate}
\item Viewing the set of buckets as a binary tree defined by the ``child" relation, we use the terms ``ancestor" and "descendant" in accordance with their standard tree definitions.
\item Let $\B_V = \{B: \exists t \in M_T \text{ s.t. }B_t = B\}$ be the set of buckets that ever produced a suboptimal payoff.
\item Let $\B_V' = \{B \in \B_V: \text{no descendant of $B$ is in $\B_V$}\}$. 
\end{enumerate}

\subsection{Proof roadmap}

The proof proceeds in the following steps:

\begin{enumerate}[leftmargin=1.5em,
    topsep=0.5ex,
    partopsep=0pt,
    parsep=0pt,
    itemsep=0.5ex]
    \item Bound the total number of buckets and therefore the total number of queries (Lemma~\ref{lem:pos-1d-queries}).
    \item Bound the suboptimality on a single time step based on the bucket length and $L$ (Lemma~\ref{lem:lipschitz-1d}).
    \item Bound the sum of bucket lengths on time steps where we make a mistake (Lemma~\ref{lem:pos-1d-bucket-lengths}), with Lemma~\ref{lem:pos-1d-bad-buckets} as an intermediate step. This captures the ``total amount of suboptimality''.
    \item \Cref{lem:pos-1d-regret} uses \Cref{lem:lipschitz-1d} and \Cref{lem:pos-1d-bucket-lengths} to bound the regret.
    \item Theorem~\ref{thm:simple} directly follows from Lemmas~\ref{lem:pos-1d-queries} and \ref{lem:pos-1d-regret}.
\end{enumerate}

\subsection{Proof}

\begin{lemma}\label{lem:pos-1d-queries}
Algorithm~\ref{alg:1d} performs at most $(\diam(\s) + 4)g(T)$ queries.
\end{lemma}

\begin{proof}
Algorithm~\ref{alg:1d} performs at most one query per bucket, so the total number of queries is bounded by the total number of buckets. There are two ways to create a bucket: from scratch (line 11), or by splitting an existing bucket (line 28). 

Since depth 0 buckets overlap only at their boundaries, and each depth 0 bucket has length $1/g(T)$, at most $g(T)\max_{t,t' \in [T]}|x_t - x_{t'}| = g(T)\diam(\s)$ depth 0 buckets are subsets of the interval $[\min_{t \in [T]} x_t,\max_{t\in[T]} x_t]$. At most two depth 0 buckets are not subsets of that interval (one at each end), so the total number of depth 0 buckets is at most $g(T) \diam(\s) + 2$.

We split a bucket $B$ when $n_B$ reaches $T/g(T)$, which creates two new buckets. Since each time step increments $n_B$ for a single bucket $B$,  and there are a total of $T$ time steps, the total number of buckets created via splitting is at most $\mfrac{2T}{T/g(T)} = 2g(T)$. Therefore the total number of buckets ever in existence is $(\diam(\s) + 2)g(T) + 2 \le (\diam(\s) + 4)g(T)$, so Algorithm~\ref{alg:1d} performs at most $(\diam(\s) + 4)g(T)$ queries.
\end{proof}

\begin{lemma}\label{lem:lipschitz-1d}
For each $t \in [T]$, $\mu_t(x_t, y_t) \ge \mu_t^m(x_t) - L \len(B_t)$.
\end{lemma}

\begin{proof}
If we query at time $t$, then $\mu_t(x_t, y_t) = \mu_t^m(x_t)$. Thus assume we do not query at time $t$: then there exists $x' \in B_t$ (as defined on line 23 of Algorithm~\ref{alg:1d}) such that $y_t = \pi(x') = \pi^m(x')$. Since $x_t$ and $x'$ are both in $B_t$, $|x_t - x'| \le \len(B_t)$. Then by local generalization, $\mu_t(x_t, y_t) = \mu_t(x_t, \pi^m(x')) \ge \mu_t^m(x_t) - L \norm{x_t-x'} \ge \mu_t^m(x_t) - L \len(B_t)$.
\end{proof}

\begin{lemma}\label{lem:pos-1d-bad-buckets}
If $\pi^m$ has at most $K$ segments, $|\B_V'| \le K$.
\end{lemma}

\begin{proof}
Now consider any $B \in \B_V'$. By definition of $\B_V'$, there exists $t \in M_T$  such that $x_t \in B$. Then there exists $x' \in B$ (as defined in Algorithm~\ref{alg:1d}) such that $y_t = \pi(x') = \pi^m(x')$. Since $t \in M_T$, we have  $ \pi^m(x_t) \ne y_t = \pi^m(x')$. Thus $x_t$ and $x'$ are in different segments, but are both in $B$. Therefore any $B \in \B_V'$ must intersect at least two segments. Since $B$ is an interval, if it intersects two segments, it must intersect two adjacent segments $X_j$ and $X_{j+1}$. Furthermore, $B$ must contain an open neighborhood centered on the boundary between $X_j$ and $X_{j+1}$.

Now consider some $B' \in \B_V'$ with $B \ne B'$. We have $|B\cap B'| \le 1$: otherwise one must be the descendant of the other, which contradicts the definition of $\B_V'$. Suppose $B'$ also intersects both $X_j$ and $X_{j+1}$: since $B'$ is also an interval, $B'$ must also contain an open neighborhood centered on the boundary between those two segments. But then $|B\cap B'| > 1$, which is a contradiction. \looseness=-1

Therefore for any pair of adjacent segments $X_j$ and $X_{j+1}$, there is at most one bucket in $\B_V'$ which contains an open neighborhood around their boundary. Since there are at most $K-1$ pairs of adjacent segments, we have $|\B_V'| \le K-1 \le K$.\looseness=-1
\end{proof}
\begin{lemma}\label{lem:pos-1d-bucket-lengths}
We have $\sum\limits_{t \in M_T} \len(B_t) \le \mfrac{2KT}{g(T)^2}$.
\end{lemma}

\begin{proof}
For every $t \in M_T$, we have $B_t = B$ for some $B \in \B_V$, so
\[
\sum_{t \in M_T} \len(B_t) = \sum_{B \in \B_V}\ \sum_{t \in M_T: B=B_t}\len(B_t)
\]
Next, observe that every $B \in \B_V \setminus \B_V'$ must have a descendant in $\B_V'$: otherwise we would have $B \in \B_V'$. Let $\mathcal{A}(B)$ denote the set of ancestors of $B$, plus $B$ itself. Then we can write
\begin{align*}
\sum_{t \in M_T} \len(B_t) \le&\ \sum_{B' \in \B_V'}\ \sum_{B \in \mathcal{A}(B')}\ \sum_{t \in M_T: B=B_t}\len(B_t)\\ 
=&\ \sum_{B' \in \B_V'}\ \sum_{B \in \mathcal{A}(B')}\ |\{t \in M_T: B=B_t\}| \cdot \len(B_t)
\end{align*}
For any bucket $B$, the number of time steps $t$ with $B = B_t$ is at most $T/g(T)$. Also recall that $\len(B) = \mfrac{1}{g(T) 2^{d(B)}}$, so
\begin{align*}
\sum_{B \in \mathcal{A}(B')} \frac{|\{t \in M_T: B=B_t\}|}{g(T) 2^{d(B)}} \le&\ \frac{T}{g(T)^2}\sum_{B \in \mathcal{A}(B')}\ \frac{1}{2^{d(B)}}\\
=&\ \frac{T}{g(T)^2} \sum_{d=0}^{d(B')} \frac{1}{2^d}\\
\le&\ \frac{T}{g(T)^2} \sum_{d=0}^{\infty} \frac{1}{2^d}\\
=&\ \frac{2T}{g(T)^2}
\end{align*}
Then by Lemma~\ref{lem:pos-1d-bad-buckets}, 
\[
\sum_{t \in M_T} \len(B_t) \le \sum_{B' \in \B_V'} \frac{2T}{g(T)^2} = \frac{2T|\B_V'|}{g(T)^2} \le \frac{2KT}{g(T)^2}
\]
as claimed.
\end{proof}

\begin{lemma}
\label{lem:pos-1d-regret}
Under the conditions of \Cref{thm:simple}, Algorithm~\ref{alg:1d} satisfies
\begin{align*}
\rmul \le&\ \frac{2LKT}{\mum g(T)^2}\\
\rplus \le&\ \frac{2LKT}{g(T)^2}
\end{align*}
\end{lemma}

\begin{proof}
We have
\begin{align*}
\rplus \le&\ \sum_{t \in M_T}(\mu_t^m(x_t) - \mu_t(x_t,y_t)) && (\text{$\mu_t(x_t,y_t)\ge \mu_t^m(x_t)$ for $t \not\in M_T$})\\
\le&\ \sum_{t \in M_T} L\len(B_t) && (\text{\Cref{lem:lipschitz-1d}})\\
\le&\ \frac{2LKT}{g(T)^2} && (\text{\Cref{lem:pos-1d-bucket-lengths}})
\end{align*}
Since $g(T) \ge 2L/\mum$ and every bucket length is at most $\frac{1}{g(T)}$, 
\begin{align*}
\mu_t(x_t,y_t) \ge&\ \mu_t^m(x_t) - L \len(B_t)\\
\ge&\ \mum - \frac{L}{g(T)}\\
\ge&\ \mum - \mum/2\\
\ge&\ \mum/2\\
>&\ 0
\end{align*}
Invoking \Cref{lem:prod-vs-add} completes the proof:
\[
\rmul \le \frac{2\rplus}{\mum} \le \frac{4LKT}{g(T)^2 \mum}
\]
\end{proof}

\Cref{thm:simple} follows from \Cref{lem:pos-1d-queries} and \Cref{lem:pos-1d-regret}.

\section{Properties of local generalization} 

\Cref{prop:lipschitz} states that Lipschitz continuity implies local generalization when the mentor is optimal.

\begin{proposition}
\label{prop:lipschitz}
Assume that $\mu$ satisfies Lipschitz continuity: for all $x,x' \in \X$ and $y \in \Y$, $|\mu(x,y) - \mu(x',y)| \le L \norm{x-x'}$. Also assume that $\mu(x, \pi^m(x)) = \max_{y \in \Y} \mu(x,y)$ for all $x \in \X$. Then $\mu$ satisfies local generalization with constant $2L$.
\end{proposition}

\begin{proof}
For any $x, x' \in \X$, we have
\begin{align*}
\mu(x, \pi^m(x')) \ge&\ \mu(x', \pi^m(x')) - L\norm{x- x'} &&\ \text{(Lipschitz continuity of $\mu$)}\\
\ge&\ \mu(x', \pi^m(x)) - L\norm{x-x'} &&\ \text{($\pi^m$ is optimal for $x'$)}\\
\ge&\ \mu(x, \pi^m(x)) - 2L\norm{x-x'} &&\ \text{(Lipschitz continuity of $\mu$ again)}\\
=&\ \mu^m(x) - 2L\norm{x-x'} &&\ \text{(Definition of $\mu^m(x)$)}
\end{align*}
Since $\pi^m$ is optimal for $x$, we have
\[
\mu^m(x) + 2L\norm{x-x'} \ge \mu^m(x) \ge \mu(x, \pi^m(x'))
\]
So $-2L\norm{x-x'} \le \mu(x,\pi^m(x')) - \mu^m(x) \le 2L\norm{x-x'}$ and therefore $|\mu(x,\pi^m(x')) - \mu^m(x)| \le 2L\norm{x-x'}$.\looseness=-1
\end{proof}

\Cref{thm:no-lg} shows that avoiding catastrophe is impossible without local generalization, even when $\s$ is $\sigma$-smooth and $\Pi$ has finite VC dimension. The first insight is that without local generalization, we can define $\mu(x,y) = \bfone(y = \pi^m(x))$ so that a single mistake causes $\prod_{t=1}^T \mu(x_t, y_t) = 0$. To lower bound $\Pr\big[\prod_{t=1}^T \mu(x_t, y_t) = 0\big]$, we use a similar approach to the proof of \Cref{thm:neg}: divide $\X = [0,1]$ into $f(T)$ independent sections with $|Q_T| << f(T) << T$, so that the agent can only query a small fraction of these sections. However, the proof of \Cref{thm:no-lg} is a bit easier, since we only need the agent to make a single mistake.

Note that \Cref{thm:no-lg} as stated only provides a bound on $\rmul$. A similar bound can be obtained for $\rplus$, but it is more tedious and we do not believe it would add much to the paper.

\begin{theorem}
\label{thm:no-lg}
Let $\X = [0,1]$ and $\Y = \{0,1\}$. Let each input be sampled i.i.d. from the uniform distribution on $\X$ and define the mentor policy class as the set of intervals within $\X$, i.e., $\Pi = \{\pi: \exists a,b \in [0,1] \text{ s.t } \pi(x) = \bfone(x \in [a,b])\ \forall x\in \X\}$. Then without the local generalization assumption, any algorithm with sublinear queries satisfies $\lim_{T\to\infty} \sup_{\bfmu, \pi^m} \E[\rmul] = \infty$.
\end{theorem}
\begin{proof}
\textbf{Part 1: Setup.} Consider any algorithm which makes sublinear worst-case queries: then there exists $g: \bbn \to \bbn$ where $\sup_{\bfmu, \pi^m} \E[|Q_T|] \le g(T)$ and $g(T) \in o(T)$. WLOG assume $g(T) \ge 0$ for all $T$; if not, redefine $g(T)$ to be $\max(g(T) , 1)$.

Define $f(T):= \lceil \sqrt{g(T)T}\rceil$; by \Cref{lem:density}, $g(T) \in o(f(T))$ and $f(T) \in o(T)$. Divide $\X$ into $f(T)$ equally sized sections $X_1,\dots,X_{f(T)}$ in exactly the same way as in \Cref{sec:neg-construction}; see also \Cref{fig:basic}. Assume that each $x_t$ is in exactly one section: this assumption holds with probability 1, so it does not affect the regret.

We use the probabilistic method: sample a segment $j^m \in [f(T)]$ uniformly at random, define $\pi^m$ by $\pi^m(x) = \bfone(x \in X_{j^m})$, and define $\mu$ by $\mu(x,y) = \bfone(y = \pi^m(x))$. In words, the mentor takes action 1 iff the input is in section $j^m$, and the agent receives payoff 1 if its action matches the mentor's and zero otherwise. Since any choice of $j^m$ defines a valid $\mu$ and $\pi^m$,
\begin{align*}
\sup_{\bfmu, \pi^m}\ \E_{\s, \A}\ [\rplus(\s, \A, \bfmu, \pi^m)] \ge \E_{j^m}\ \E_{\s, \A}\ [\rplus(\s, \A, (\mu,\dots,\mu), \pi^m)]
\end{align*}
Let $J_{\neg Q} = \{j \in [f(T)]: x_t \not\in X_j\ \forall t \in Q_T\}$ be the set of sections which are never queried. Let $j_1,\dots,j_k$ be the sequence of sections queried by the agent: then $k = |Q_T| \le g(T)$. 

\textbf{Part 2: The agent is unlikely to determine $j^m$.} By the chain rule of probability,
\begin{align*}
\Pr[j^m \in J_{\neg Q}] = \Pr\big[j_i \ne j^m\ \forall i \big] = \prod_{i=1}^{k} \Pr\big[j_i \ne j^m \mid j_r \ne j^m\ \forall r < i\big]
\end{align*}
Now fix $i$ and assume $j_r \ne j^m\ \forall r  < i$. Queries in sections other than $j^m$ provide no information about the value of $j^m$, so $j^m$ is uniformly distributed across the set of sections not yet queried, i.e., $\{j \in [f(T)]: j_r \ne j\ \forall r < i\}$. There are at least $f(T) - i + 1$ such sections, since there are $i-1$ prior queries at this point. Thus $\Pr[j_i \ne j^m \mid j_r \ne j^m\ \forall r < i] \ge \frac{f(T) -i}{f(T)-i + 1}$ (the inequality is because this probability could also be 1 if $j_i = j_r$ for some $i < r$). Therefore
\begin{align*}
\Pr\big[j^m \in J_{\neg Q} \big] \ge&\ \prod_{i=1}^{k} \frac{f(T) -i }{f(T)-i+1}\\
=&\ \frac{f(T)-1}{f(T)} \cdot \frac{f(T)-2}{f(T)-1} \dots \frac{f(T)-k + 1}{f(T)-k+2} \cdot \frac{f(T)-k}{f(T)-k+1}\\
=&\ \frac{f(T) - k}{f(T)}\\
\ge&\ 1 - \frac{g(T)}{f(T)}
\end{align*}
\textbf{Part 3: If the agent fails to determine $j^m$, it is likely to make at least one mistake.} For each $j \in J_{\neg Q}$, let $V_j = \{t \in [T]: x_t \in X_j\}$ be the set of time steps with inputs in section $j$. By \Cref{lem:concentrate}, $\Pr[|V_{j^m}| = 0] \le \exp\big(\frac{T}{16f(T)}\big)$. Then by the union bound, $\Pr[j^m \in J_{\neg Q} \text{ and } |V_{j^m}| > 0] \ge 1 - \frac{g(T)}{f(T)} - \exp\big(\frac{-T}{16f(T)}\big)$. For the rest of Part 3, assume $j^m \in J_{\neg Q} \text{ and } |V_{j^m}| > 0$. 

\emph{Case 1:} For all $j \in J_{\neg Q}$ and $t \in V_j$, we have $y_t = 0$. In particular, this holds for $j = j^m$, and we know there exists at least one $t \in V_{j^m}$ since $|V_{j^m}| > 0$. Then $y_t \ne \pi^m(x_t)$, so $\mu(x_t, y_t) = 0$ and thus $\Pr\left[\prod_{r=1}^T \mu(x_r,y_r) = 0 \: \Big| \: j^m \in J_{\neg Q} \text{ and } |V_{j^m}| > 0 \text{ and } y_t = 0\ \forall j \in J_{\neg Q}, t\in V_j \right] = 1$.

\emph{Case 2:} There exists $j \in J_{\neg Q}$ and $ t \in V_j$ with $y_t = 1$. Then $\mu(x_t, y_t ) = 0$ unless $j = j^m$, so
\begin{align*}
&\ \Pr\left[\prod_{r=1}^T \mu(x_r,y_r) = 0 \: \Big| \:j^m \in J_{\neg Q} \text{ and } |V_{j^m}| >0 \text{ and } \exists j \in J_{\neg Q}, t\in V_j \text{ s.t. } y_t = 1 \right]\\
\ge&\ \Pr\Big[\mu(x_t, y_t) = 0\ \Big|\ j^m \in J_{\neg Q} \text{ and } |V_{j^m}| >0\text{ and } \exists j \in J_{\neg Q}, t\in V_j \text{ s.t. } y_t = 1\Big]\\
=&\ \Pr\Big[j \ne j^m \ \Big|\ j^m \in J_{\neg Q} \text{ and } |V_{j^m}| >0\text{ and } \exists j \in J_{\neg Q}, t\in V_j \text{ s.t. } y_t = 1\Big]
\end{align*}
Since $j^m \in J_{\neg Q}$, the agent has no information about $j^m$ other than that it is in $J_{\neg Q}$. This means that $j^m$ is uniformly distributed across $J_{\neg Q}$, so
\begin{align*}
\Pr\left[\prod_{r=1}^T \mu(x_r,y_r) = 0 \: \Big| \:j^m \in J_{\neg Q} \text{ and } |V_{j^m}| >0\text{ and } \exists j \in J_{\neg Q}, t\in V_j \text{ s.t. } y_t = 1\right] \ge 1 - \frac{1}{|J_{\neg Q}|} \ge 1 - \frac{1}{f(T) - g(T)}
\end{align*}
Combining Case 1 and Case 2, we get the overall bound of
\[
\Pr\left[\prod_{t=1}^T \mu(x_t,y_t) = 0 \: \Big| \: j^m \in J_{\neg Q} \text{ and } |V_{j^m}| > 0 \right] \ge 1 - \frac{1}{f(T) - g(T)}
\]
and thus
\begin{align*}
 \Pr\left[\prod_{t=1}^T \mu(x_t,y_t) = 0 \right] \ge&\ \Pr\left[\prod_{t=1}^T \mu(x_t,y_t) = 0 \text{ and } j^m \in J_{\neg Q} \text{ and } |V_{j^m}| > 0 \right]\\
=&\ \Pr\left[\prod_{t=1}^T \mu(x_t,y_t) = 0 \: \Big| \: j^m \in J_{\neg Q} \text{ and } |V_{j^m}| > 0 \right] \cdot \Pr\Big[j^m \in J_{\neg Q} \text{ and } |V_{j^m}| > 0 \Big]\\
\ge&\ \left(1 - \frac{1}{f(T) - g(T)}\right) \left(1 - \frac{g(T)}{f(T)} - \exp\left(\frac{-T}{16f(T)}\right)\right)
\end{align*}
For brevity, let $\alpha(T)$ denote this final bound. Since $g(T) \in o(f(T))$ and $f(T) \in o(T)$, we have 
\begin{align*}
\lim_{T\to\infty} \alpha(T) = \lim_{T\to\infty} \left(1 - \frac{1}{f(T) - g(T)}\right) \left(1 - \frac{g(T)}{f(T)} - \exp\left(\frac{-T}{16f(T)}\right)\right)
= (1-0)(1-0 -0)
= 1
\end{align*}
\textbf{Part 4: Putting it all together.} Consider any $\ep  \in (0,1]$; to avoid dealing with infinite expectations, we will deal with $\Pr[\prod_{t=1}^T \mu(x_t,y_t) \le \ep]$ instead of $\Pr[\prod_{t=1}^T \mu(x_t,y_t) = 0]$. Since $\prod_{t=1}^T \mu(x_t, y_t) \le 1$ always, we have
\begin{align*}
\E_{j^m}\ \E_{\s, \A}\ \left[\log\prod_{t=1}^T \mu(x_t,y_t) \right] =&\ \E_{j^m}\ \E_{\s, \A}\ \left[\log \prod_{t=1}^T \mu(x_t,y_t) \: \Big| \: \prod_{t=1}^T \mu(x_t,y_t) \le \ep \right] \cdot \Pr\left[\prod_{t=1}^T \mu(x_t,y_t) \le \ep \right]\\
 +&\ \E_{j^m}\ \E_{\s, \A}\ \left[\log \prod_{t=1}^T \mu(x_t,y_t) \: \Big| \: \prod_{t=1}^T \mu(x_t,y_t) > \ep \right] \cdot \Pr\left[\prod_{t=1}^T \mu(x_t,y_t) > \ep \right]\\
\le&\log \ep   \cdot \Pr\left[\prod_{t=1}^T \mu(x_t,y_t) \le \ep \right]+ 1 \cdot \left(1 - \Pr\left[\prod_{t=1}^T \mu(x_t,y_t) \le \ep \right]\right)\\
\le&\ 1 - (1 - \log \ep ) \Pr\left[\prod_{t=1}^T \mu(x_t,y_t) \le \ep \right]
\end{align*}
Since $\ep \in (0,1]$, we have  $1 - \log \ep  > 0$. Also, $\Pr[\prod_{t=1}^T \mu(x_t,y_t) \le \ep] \ge \Pr[\prod_{t=1}^T \mu(x_t,y_t) = 0]$, so
\begin{align*}
\E_{j^m}\ \E_{\s, \A}\ \left[\log\prod_{t=1}^T \mu(x_t,y_t) \right] \le&\ 1 - (1 - \log \ep ) \Pr\left[\prod_{t=1}^T \mu(x_t,y_t) \le \ep \right]\\
\le&\ 1 - (1 - \log \ep ) \Pr\left[\prod_{t=1}^T \mu(x_t,y_t) =0  \right]\\
\le&\ 1 - (1 - \log \ep ) \alpha(T) 
\end{align*}
Since $\prod_{t=1}^T \mu^m(x_t) = 1$ always, we have
\begin{align*}
\sup_{\bfmu, \pi^m}\ \E_{\s, \A}\ [\rplus(\s, \A, \bfmu, \pi^m)] \ge&\ \E_{j^m}\ \E_{\s, \A}\ [\rplus(\s, \A, (\mu,\dots,\mu), \pi^m)]\\
=&\ \log 1 - \E_{j^m} \E_{\s, \A} \left[\log \prod_{t=1}^T \mu(x_t,y_t) \right]\\
\ge&\ - 1 + (1 - \log \ep ) \alpha(T)
\end{align*}
Therefore
\begin{align*}
\lim_{T\to\infty} \sup_{\bfmu, \pi^m}\ \E_{\s, \A}\ [\rplus(\s, \A, \bfmu, \pi^m)] \ge&\ - 1 + (1 - \log \ep ) \lim_{T\to\infty} \alpha(T)\\
\ge&\ - 1 + (1 - \log \ep )\\
\ge&\ -\log \ep
\end{align*}
This holds for every $\ep \in (0,1]$, which is only possible if $\lim_{T\to\infty} \sup_{\bfmu, \pi^m}\ \E_{\s, \A}\ [\rplus(\s, \A, \bfmu, \pi^m)] = \infty$, as desired.
\end{proof}

\end{document}